\documentclass{article}

\usepackage[T1]{fontenc}
\usepackage[utf8]{inputenc}
\usepackage[british]{babel}
\usepackage{lmodern}
\usepackage{microtype}

\usepackage{textcomp}
\usepackage{textgreek}
\usepackage{caption}
\DeclareUnicodeCharacter{2032}{\textprime}

\usepackage[a4paper,margin=1in]{geometry}

\usepackage{authblk}

\usepackage{newunicodechar}
\usepackage{amsmath, amssymb, amsthm}
\usepackage{tikz-cd}
\usepackage{geometry}
\usepackage{mathtools}
\usepackage{pifont} % \ding{51} for checkmark (we also use \checkmark)

\usepackage{enumitem}

\usepackage{xcolor}

\usepackage{graphicx}
\usepackage[hypcap=false]{caption}

\usepackage{hyphenat}

\usepackage{csquotes}

\usepackage[
  backend=biber,
  style=authoryear,
  language=british
]{biblatex}
\usepackage[
  unicode=true,
  pdfencoding=auto,
  pdfborder={0 0 0},
  hidelinks
]{hyperref}

\usepackage{amsthm}

\newtheorem{proposition}{Proposition}[section]

\title{A Categorical Analysis of Large Language Models and Why LLMs Circumvent the Symbol Grounding Problem}
\author{Luciano Floridi\textsuperscript{1,2,*}, 
Yiyang Jia\textsuperscript{3,*},
Fernando Tohmé\textsuperscript{4,5}}

\date{} % Explicitly clear the date

\begin{document}

\maketitle

\noindent\textsuperscript{1}Digital Ethics Center, Yale University, 85 Trumbull Street, New Haven, CT 06511, US\\
\textsuperscript{2}Department of Legal Studies, University of Bologna, Via Zamboni, 27/29, 40126, Bologna, IT\\
\textsuperscript{3}Department of Information Systems, Tokyo City University, 3-3-1 Ushikubo-Nishi,Tsuzuki-Ku,Yokohama,
Kanagawa 224-8551, JP\\
\textsuperscript{4}Departamento de Economía - Universidad Nacional del Sur, Ar\\
\textsuperscript{5}Instituto de Matemática de Bahía Blanca - CONICET, Ar\\
\textsuperscript{*}These authors contributed equally as the first co-authors.

\noindent Email for correspondence: luciano.floridi@yale.edu

\begin{abstract}
\noindent 
This paper presents a formal, categorical framework for analysing how humans and large language models (LLMs) transform content into truth-evaluated propositions about a state space of possible worlds \(W\), in order to argue that LLMs do not solve but circumvent the symbol grounding problem. Operating at an epistemological level of abstraction within the category of relations ($\mathbf{Rel}$), we model the human route \((H \to C \to \mathrm{Pred}(W))\)---consultation and interpretation of grounded content---and the artificial route \((H \to C' \to G \times C' \to O \to \mathrm{Pred}(W))\)---prompting a trained LLM and interpreting its outputs---together with the training pipeline \((C \to C' \to D(C') \to G)\). The framework distinguishes syntax from semantics, represents meanings as propositions within $\mathrm{Pred}(W)$ (the power set of $W$), and defines success as \textit{soundness (entailment)}: the success set $H^{\checkmark}_{\subseteq}$ where the AI's output \textit{set} $P_{\mathrm{AI}}(h)$ is a \textit{subset} of the human ground-truth set $P_{\mathrm{human}}(h)$. We then locate \textit{failure modes} at tokenisation, dataset construction, training generalisation, prompting ambiguity, inference stochasticity, and interpretation. On this basis, we advance the central thesis that LLMs lack unmediated access to \(W\) and therefore do not solve the symbol grounding problem. Instead, they circumvent it by exploiting pre-grounded human content. We further argue that apparent semantic competence is \textit{derivative}  of human experience, causal coupling, and normative practices, and that hallucinations are entailment failures ($P_{\mathrm{AI}}(h) \not\subseteq P_{\mathrm{human}}(h)$), which are intrinsic to this architecture, not mere implementation bugs. The categorical perspective clarifies debates clouded by anthropomorphic language, connects to extensions (e.g., probabilistic morphisms, partiality for refusals), and delineates the boundaries within which LLMs can serve as reliable epistemic interfaces. We discuss idealisations and scope limits, and conclude with some methodological guidance: expand $H^{\checkmark}_{\subseteq}$ through curation, tooling, and verification, while avoiding attributing any understanding to stochastic, pattern-completing systems.
\end{abstract}

\noindent\textbf{Keywords}: Artificial intelligence; Category Theory; Hallucinations; Large Language Models; Levels of Abstraction; Symbol Grounding Problem

\medskip 

\noindent\textbf{Conflict of interest}: the authors declare that they have no conflict of interest.

\medskip 

\noindent\textbf{Data availability}: not applicable.

\medskip 

\noindent\textbf{Acknowledgements}: we would like to thank Kate Boxer, Jérémie Koenig, Vincent Wang-Maścianica, Jessica Morley, Claudio Novelli, Jeffrey W Sanders, and David Watson for their helpful input, suggestions, and feedback on previous versions of this article. They are responsible only for its improvements and not for any remaining shortcomings. 

\section{Introduction}
\label{sec:introduction}

In this article, we present two main results, using the \textit{method of levels of abstraction} \parencite{floridi_method_2008} and \textit{category theory} (\textcite{awodey_category_2010, landry_categories_2017, patterson_knowledge_2017, kornell_axioms_2023}).

First, we model\footnote{In this article, we have tried to avoid the confusion between model as a verb, `model' as the outcome of the categorical modelling, and `model' as a synonym of Large Language Model. For the sake of clarity, we have used `analysis' and `framework' whenever the use of `model' could have been ambiguous. In some cases, such as in `model-theoretic semantics', we have relied on context to disambiguate the expression.} the epistemic relationship between human knowledge and Large Language Models (LLMs), specifying precisely when an LLM's output reaches the conclusions that an ideally well-informed human agent could reach. Second, we use this analysis to argue that LLMs circumvent rather than solve the symbol grounding problem (\parencite{harnad_symbol_1990}, \parencite{harnad_language_2024}).

The adoption of an appropriate level of abstraction (LoA) constitutes a fundamental methodological choice that determines which observables are considered relevant while systematically excluding extraneous details, including, but not limited to, implementation details. Our analysis operates at an epistemologically abstract yet informative level, treating humans, content repositories, and information (formalised as propositions) about \textit{possible worlds}\footnote{We use the standard Kripkean framework for modal semantics \parencite{kripke_naming_1980}.} as primary analytical components, while deliberately abstracting from the neurophysiological characteristics of human cognition and the statistical features of LLM architectures. This methodological approach follows the established insight that complex phenomena can be rigorously explained by analysing simpler, systematically interrelated components operating at an appropriate level of theoretical description \parencite{marr_vision_1982, floridi_method_2008}.\footnote{David Marr's influential tripartite analysis of levels---computational, algorithmic, and implementational---in cognitive science \parencite{marr_vision_1982} profoundly shaped contemporary approaches to understanding complex systems. The present framework primarily operates at Marr's computational and algorithmic levels for the human--AI system, systematically ignoring the implementational details of neural circuits or silicon architectures. This analytical choice is also consistent with Allen Newell's conception of the `knowledge level' in artificial intelligence \parencite{newell_knowledge_1982}, wherein an agent is characterised by the knowledge it possesses and the goals it pursues, abstracting entirely from how such knowledge is internally represented or processed.} Through explicit specification of this epistemological LoA, we avoid the error of conflating a system's symbolic operations with the semantic content of its productions, a fundamental mistake warned against in arguments such as Searle's Chinese Room thought experiment \parencite{searle_minds_1980} and Harnad's formulation of the symbol grounding problem \parencite{harnad_symbol_1990,taddeo_solving_2005,taddeo_praxical_2007}. Instead, the categorical analysis focuses on the systematic ways in which information is produced, used, and interpreted by human agents and LLMs to generate knowledge claims about some states of affairs (possible worlds). At this specified LoA, a human agent (a user characterised by particular background knowledge and epistemic goals) and an artificial agent (an LLM) are analysed as two distinct yet comparable epistemic agents that interface with corpora of informational content. Both agents consume content and produce statements that constitute claims about possible worlds. Figure~\ref{fig:1} summarises the components and processes operative in the different routes, which we will explain in detail in Section~\ref{sec:components}. By examining this categorical diagram, we seek to clarify the conditions under which LLMs can effectively function as informational interfaces that mediate between users and the vast repository of accumulated human content \parencite{floridi_ai_2023}, and to characterise the circumstances under which such systems fail to perform this mediating function.

The remainder of the article is structured as follows. We develop the categorical analysis in Sections~\ref{sec:components} and \ref{sec: categorical_analysis}, where we provide a side-by-side formal comparison of the human epistemic pathways and the LLM processing pathway from the initial query to the asserted answer. We refine the analysis in Sections~\ref{sec:refinements} and \ref{sec:sources of mismatch}, and provide a simple illustration in Section~\ref{sec:example}. In Section~\ref{sec:lessons}, we move to the discussion of two philosophical consequences of the analysis. The second main result is articulated in Section~\ref{sec:SGP}. There, we argue that LLMs do not solve the symbol grounding problem but instead circumvent it by building on previous analysis and discussing the current literature on the topic. In Section~\ref{sec:limitations}, we discuss some limitations of our approach. We draw some conclusions in Section~\ref{sec:conclusion}. In the Appendix 
\ref{sec:appendix}, we provide some necessary mathematical definitions to interpret the formal framework presented in the article.

\section{Categorical Framework}
\label{sec:components}
Our analysis is carried out in the category of relations, $\mathbf{Rel}$. Its objects are sets, and its morphisms $R\colon A\to B$ are relations $R\subseteq A\times B$, with identities given by diagonal relations and composition by relational composition. More precisely, $\mathbf{Rel}$ is defined as follows (see also Appendix A)\footnote{For introductions to Applied Category Theory, see \textcite{spivak_category_2014}, \textcite{fong_invitation_2019}, and \textcite{perrone_notes_2019}, now \textcite{perrone_starting_2024}.}

\begin{itemize}
    \item \textbf{Objects:} The objects are sets (the same as the objects in \textbf{Set}).
    \item \textbf{Morphisms:} A morphism $R: A \to B$ from a set $A$ to a set $B$ is a relation defined as a subset of the Cartesian product $R \subseteq A \times B$.
    \item \textbf{Identity:} The identity morphism $\mathrm{id}_A: A \to A$ is the diagonal relation $\mathrm{id}_A = \{(a, a) \mid a \in A\}$.
    \item \textbf{Composition:} Given two relations $R: A \to B$ and $S: B \to C$, their composite $S \circ R: A \to C$ is defined as:
    \[
        S \circ R = \{(a, c) \in A \times C \mid \exists b \in B, (a, b) \in R \text{ and } (b, c) \in S \}
    \]
\end{itemize}
	
The choice of \(\mathbf{Rel}\) is explicitly motivated by the fact that many of our key morphisms (such as interpretation, prompting, and evaluation) are `one-to-many' or inherently stochastic, and \(\mathbf{Rel}\) provides the mathematical rigour to handle these processes uniformly.

Within \(\mathbf{Rel}\), our framework consists of a subcategory $\mathcal{C}$, in which the objects and their characterisations are:

\begin{itemize}[leftmargin=2em]
    \item \(H\): human epistemic situations (background knowledge, goals, query context).
    \item \(C\): human-authored content, i.e.\ syntactically well-formed and meaningful text.
    \item \(C^{\prime}\): tokenised strings, the machine-readable version of content and prompts.
    \item \(D(C^{\prime})\): datasets constructed from \(C^{\prime}\), abstracting from concrete sampling and cleaning procedures.
    \item \(G\): space of trained LLMs, with each \(g\in G\) representing a particular set of weights.
    \item \(O\): LLM outputs, as sequences of tokens.
    \item \(W\): space of (relevant) possible worlds or states of affairs.
    \item \(\mathrm{Pred}(W)=\mathcal{P}(W)\): the powerset of \(W\), seen as the space of propositions about \(W\). \(\mathrm{Pred}(W)\) is regarded as a poset under inclusion, which allows us to use set inclusion \(\subseteq\) to define entailment and compare the soundness of the human and AI routes.
\end{itemize}
	
\noindent Figures~\ref{fig:1} and \ref{fig:2} summarise the category $\mathcal{C}$.

\begin{figure}[htbp]
\centering
\includegraphics[width=.7\linewidth]{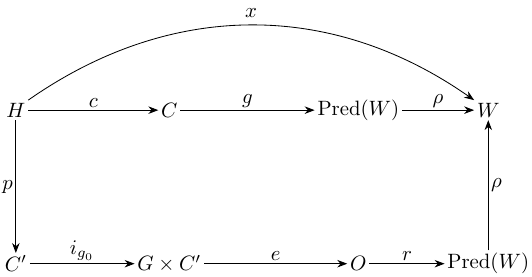}
\caption{Schematic representation of the human and LLM routes in the category $\mathcal{C}$, a subcategory of $\mathbf{Rel}$. The diagram illustrates three possible mappings from the space of human epistemic states $H$ to the state space of possible worlds $W$.}
\label{fig:1}
\end{figure}

\begin{figure}[htbp]
\centering
\includegraphics[width=.7\linewidth]{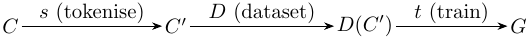}
\caption{The training pipeline: from content to the space of trained models.}
\label{fig:2}
\end{figure}

\noindent These diagrams involve the morphisms in $\mathcal{C}$. Apart from the identities, these morphisms are defined as follows (with their interpretations in parentheses):

\begin{itemize}[leftmargin=*]
\item[] $x: H \to W$ (experience)
\item[] $c: H \to C$ (consult content)
\item[] $g: C \to \mathrm{Pred}(W)$ (interpret content)
\item[] $p: H \to C^{\prime}$ (generate prompt)
\item[] $s: C \to C^{\prime}$ (tokenise)
\item[] $D: C^{\prime} \to D(C^{\prime})$ (construct dataset)
\item[] $t: D(C^{\prime}) \to G$ (train to obtain a model)
\item[] $i_{g_0}: C^{\prime} \to G \times C^{\prime}$ (pair with a fixed model $g_0$). This is a deterministic morphism (a function) given by the relation $\{(c', (g_0, c')) \mid c' \in C'\}$.
\item[] $e: G \times C^{\prime} \to O$ (evaluate)
\item[] $r: O \to \mathrm{Pred}(W)$ (assign semantics; this is the user's interpretation step)
\item[] $\rho: \mathrm{Pred}(W) \to W$ (resolve reference; indicates that $\mathrm{Pred}(W)$ is about $W$)
\end{itemize}

\noindent We assume a non-standard notion of commutativity for the diagrams, which we call \emph{entailment-commutativity} (or \textit{lax commutativity}). This implies that the AI path is contained within the human path globally:
\[
\forall h \in H, \quad
(r \circ e \circ i_{g_0} \circ p)(h) \subseteq (g \circ c)(h)
\quad\text{instead of}\quad
(r \circ e \circ i_{g_0} \circ p)(h) = (g \circ c)(h).
\]

\noindent This shift replaces the usual requirement of path equality with the inclusion of relations. It can be better understood by conceiving $\mathcal{C}$ as a $2$-category, inheriting the properties of $\mathbf{Rel}$. The cells are as follows:

\begin{itemize}
    \item \textit{$0$-cells (Objects):} Sets (e.g., $H$, $C^{\prime}$, $\mathrm{Pred}(W)$).
    \item \textit{$1$-cells (Morphisms):} Relations (e.g., $p$, $g \circ c$).
    \item \textit{$2$-cells (Transformations):} Inclusions. Given two parallel relations $R, S: A \to B$, a $2$-morphism $R \Rightarrow S$ exists if and only if $R \subseteq S$ (i.e., $R$ is a subset of $S$ in $A \times B$).
\end{itemize}

\noindent In this context, the {\it entailment-commutativity} condition is precisely a statement about the existence of a specific $2$-cell. We can characterise this property using the universal property of the \textit{right Kan extension}.

\noindent Consider the composite relation $g \circ c: H \to \mathrm{Pred}(W)$ and the prompting relation $p: H \to C^{\prime}$. We seek to characterise the optimal path from $C^{\prime}$ to $\mathrm{Pred}(W)$ that serves as the soundness benchmark for the AI route $r \circ e \circ i_{g_0}$.

\noindent The \textit{right Kan extension} of $g\circ c$ along $p$, denoted as $\mathrm{Ran}_p(g\circ c)$, is defined by a universal property in $\mathbf{Rel}$. 
Formally, as a morphism in $\mathbf{Rel}$, it can be constructed as the union of all sound relations:

$$\mathrm{Ran}_p(g \circ c) = \bigcup \{ R \subseteq C' \times \mathrm{Pred}(W) \mid R \circ p \subseteq g \circ c \}$$

\noindent Diagrammatically:
	
	\begin{center}
		\begin{tikzcd}[row sep=large, column sep=large]
			H \arrow[rr, "g \circ c" {name=Human}] \arrow[d, "p"'] & & \text{Pred}(W) \\
			C' \arrow[urr, "\text{Ran}_p(g \circ c)"' {name=AI}, dashed] & &
			\arrow[Rightarrow, from=AI, to=Human, "\subseteq" description]
		\end{tikzcd}
\end{center}

\noindent However, its operational behaviour is best understood pointwise. For any specific prompt $c' \in C'$, the value of this maximal relation is computed as the limit (intersection) of the ground-truths for all compatible epistemic situations:

%$$\mathrm{Ran}_p(g \circ c)(c') = \bigcap \{ P_{\mathrm{human}}(h) \mid h \in H, p(h) = c' \} = \bigcap_{h \in p^{-1}(c')} P_{\mathrm{human}}(h)$$

\[
\mathrm{Ran}_p(g \circ c)(c') = \bigcap \{ (g \circ c)(h) \mid h \in H, p(h) = c' \} = \bigcap_{h \in p^{-1}(c')} (g \circ c)(h)
\]

\noindent This pointwise intersection captures the core intuition of soundness: to be safe under ambiguity, the AI must output only the common truths shared by all interpretations.

In the diagram of Figure~\ref{fig:1}, $g \circ c$ and $r \circ e \circ i_{g_0}$ are the {\em human} and the LLM routes, respectively. That is:

\begin{itemize}
\item The upper human route is modelled as a composite relation
	\[
	g\circ c \colon H \longrightarrow \mathrm{Pred}(W),
	\]
	where:
	\begin{itemize}[leftmargin=2em]
		\item \(c\colon H\to C\) represents the consultation of relevant content given an epistemic situation \(h\in H\).
		\item \(g\colon C\to \mathrm{Pred}(W)\) maps content to plausible propositions, allowing multiple, possibly inconsistent readings.
	\end{itemize}
	For each \(h\in H\), the set
	\[
	P_{\mathrm{human}}(h)\coloneqq (g\circ c)(h)\subseteq \mathrm{Pred}(W)
	\]
    
	is defined as the human ground-truth benchmark for that epistemic situation, namely the set of propositions a suitably idealised human could draw from consultation and interpretation of reliable content.
	
	In addition, an experience arrow \(x\colon H\to W\) represents direct human coupling with the world (perception, interaction).\footnote{This arrow does not appear in the LLM route and plays a mainly conceptual role in the symbol-grounding discussion.}

\item 	The lower LLM route is another composite in \(\mathbf{Rel}\),
	\[
	r\circ e\circ i_{g_{0}}\circ p \colon H \longrightarrow \mathrm{Pred}(W),
	\]
	where:
    
	\begin{itemize}[leftmargin=2em]
		  \item \(p\colon H\to C^{\prime}\) associates possible tokenised prompts with an epistemic situation.
		\item \(i_{g_{0}}\colon C^{\prime}\to G\times C^{\prime}\) is a deterministic pairing with a fixed trained model \(g_{0}\in G\).        
		\item \(e\colon G\times C^{\prime}\to O\) is the (stochastic) evaluation relation that models sampling during inference.
		\item \(r\colon O\to \mathrm{Pred}(W)\) interprets the output strings as sets of plausible propositions, again allowing ambiguity.
	\end{itemize}
\noindent	For each \(h\in H\), the LLM route yields a set
	\[
	P_{\mathrm{AI}}(h)\coloneqq (r\circ e\circ i_{g_{0}}\circ p)(h) \subseteq \mathrm{Pred}(W),
	\]    
\noindent which is the set of propositions ascribed (by a human interpreter) to the LLM output for that query.
\end{itemize}

\noindent The offline training pipeline depicted in Figure~\ref{fig:2} is also packaged as a composite relation
	\[
	t\circ D\circ s \colon C \longrightarrow G,
	\]
	where \(s\colon C\to C^{\prime}\) tokenises content, \(D\colon C^{\prime} \to D(C^{\prime})\) constructs datasets, and \(t\colon D(C^{\prime})\to G\) models (stochastic) training, so that a fixed \(g_{0}\) is one value of \(t(d)\) for some dataset \(d\in D(C^{\prime})\).

\noindent The commutativity of the diagram in Figure~\ref{fig:1} is formulated as an entailment relation in \(\mathrm{Pred}(W)\). We say that for each \(h\in H\), the LLM route is sound (non-hallucinatory) if
	\[
	P_{\mathrm{AI}}(h)\subseteq P_{\mathrm{human}}(h),
	\]
and the \emph{soundness set} is defined by
	\[
	H^{\ast}\coloneqq\{\,h\in H \mid P_{\mathrm{AI}}(h)\subseteq P_{\mathrm{human}}(h)\,\}.
	\]

\noindent This soundness property can be naturally reframed using the right Kan extension. Recall that, as defined above, the pointwise value of the right Kan extension is the intersection of ground truths.
    
\noindent For any specific $h \in H$, since $h$ is one of the possible situations that generate the prompt (i.e., $h \in p^{-1}(p(h))$), the intersection of all such situations must be a subset of the specific component $P_{\mathrm{human}}(h)$. Therefore:
   \[
   \text{Ran}_p(g \circ c)(p(h)) \subseteq P_{\mathrm{human}}(h)
   \]

\noindent Then, we have:
\begin{proposition}
    \label{prop:kan_soundness}
    For any $h \in H$:
    \begin{enumerate}
        \item If $P_{\mathrm{AI}}(h) \subseteq \mathrm{Ran}_p(g \circ c)(p(h))$, then $h \in H^{\ast}$.
        \item If the morphism $p$ is an injective function, then $h \in H^{\ast}$ implies $P_{\mathrm{AI}}(h) \subseteq \mathrm{Ran}_p(g \circ c)(p(h))$.
    \end{enumerate}
\end{proposition}

\begin{proof}
    \begin{enumerate}
\item Assume that $P_{\mathrm{AI}}(h) \subseteq \mathrm{Ran}_p(g \circ c)(p(h))$. Recall that by definition, the intersection is a subset of each component, so $\mathrm{Ran}_p(g \circ c)(p(h)) \subseteq P_{\mathrm{human}}(h)$. By transitivity, we have $P_{\mathrm{AI}}(h) \subseteq P_{\mathrm{human}}(h)$, and thus $h \in H^{\ast}$.

\item Assume that $h \in H^{\ast}$, i.e., $P_{\mathrm{AI}}(h) \subseteq P_{\mathrm{human}}(h)$. We need to prove that:
        \[ P_{\mathrm{AI}}(h) \subseteq \bigcap_{h' \in p^{-1}(p(h))} P_{\mathrm{human}}(h') \]

\noindent If we assume that $p$ is an injective function, the fibre $p^{-1}(p(h))$ is the singleton $\{h\}$. In this case:
        \[ \mathrm{Ran}_p(g \circ c)(p(h)) = \bigcap_{x \in \{h\}} P_{\mathrm{human}}(x) = P_{\mathrm{human}}(h) \]
        \noindent Since $P_{\mathrm{AI}}(h) \subseteq P_{\mathrm{human}}(h)$, it follows directly that $P_{\mathrm{AI}}(h) \subseteq \mathrm{Ran}_p(g \circ c)(p(h))$.
\end{enumerate}
\end{proof}

\noindent In the general case of ambiguous prompting, the right Kan extension is a \emph{stricter} condition than simple soundness: it requires the LLM to be sound across all possible interpretations of the prompt, not just the current one.

A stricter \emph{world-level agreement set} \(H^{!} \subseteq H\) is defined by requiring that, after applying a resolver morphism \(\rho\colon \mathrm{Pred}(W)\to W\) (designating actual worlds where propositions hold), both routes select the same unique world.

Formally, if \(W_{\mathrm{human}}(h)\) and \(W_{\mathrm{AI}}(h)\) are the corresponding sets of worlds, then
	\[
	H^{!}\coloneqq\{\,h\in H\mid W_{\mathrm{human}}(h)=W_{\mathrm{AI}}(h)=\{w\}\text{ for some }w\in W\,\}.
	\]

\noindent Conceptually, we expect that \(H^{!} \subseteq H^{\ast}\), meaning that achieving strict world-level agreement implies soundness (non-hallucination). Thus, hallucinations in LLMs can be characterised as \textit{failures of entailment}. That is, given \(h\in H\), a proposition \(P\in \mathrm{Pred}(W)\) is a hallucination if \(P\in P_{\mathrm{AI}}(h)\) but \(P\notin P_{\mathrm{human}}(h)\). Consequently, the system is in a hallucinatory state for a situation \(h\) if and only if \(h\notin H^{\ast}\).

Propositions can be non-hallucinatory yet still not true. To connect these two concepts, we start by designating a specific world \(w^\ast \in W\) as the ``actual world''. The truth of a proposition \(P \in \mathrm{Pred}(W)\) (i.e., \(P \subseteq W\)) at this designated world \(w^\ast\) can be defined according to the standard Kripkean satisfaction relation \(\Vdash\):

\[
w^\ast \Vdash P \;\Longleftrightarrow\; w^\ast \in P.
\]

\noindent Then, given the \textit{soundness success set} \(H^{\ast}\) (the set of all \(h\) where no hallucinations occur), if an \(h \in H^{\ast}\) leads to a proposition \(P\) that is false at \(w^\ast\), it implies that \(P\) must have already been present in the human ground-truth set \(P_{\mathrm{human}}(h)\).\footnote{Note that this means treating hallucinations as equivalent to \textit{unwarranted} yet not simply \textit{untrue} statements; that is, as any output that is not warranted by the evidence available to the system for the task at hand---i.e., it is unsupported by the provided context, retrievable or referenced sources, or the input signal---and therefore fails to correspond to what is known or stipulated to be true for that task. This includes content that is \textit{fabricated} (introduced without an evidential basis), \textit{contradicts} the evidence provided, or \textit{extends} beyond the scope of the input in ways that cannot be justified by task norms. This distinction clarifies that hallucinations may still be true given a different benchmark (e.g., accidental truth).}

To analyse in more detail the distinction between non-hallucinatory and true propositions, we have to distinguish the following concepts:

\begin{itemize}

\item \textit{Data}: sequences of machine-readable tokens. Here, the elements of $C^{\prime}$.

\item \textit{Content}: data sets that are syntactically well-formed and semantically meaningful. We identify them here with the image of the evaluation morphism $\mathrm{Im}(e) \subseteq O$.

\item \textit{Information}: \textit{truthful} content. Elements $\bar{e} \in \mathrm{Im}(e) \subseteq O$ such that the actual world $w^{\ast}$ is compatible with their interpretation (i.e., $w^{\ast} \in (\rho \circ r)(\bar{e})$).
\item \textit{Knowledge}: explained, justified, warranted, or otherwise supported information in non-Gettierised and non-sceptical contexts.
   
\end{itemize}

\noindent In light of these distinctions, we can incorporate into our analysis the upper arrow of Figure~\ref{fig:1}, which does not intervene in our analysis of hallucinations. That is, the morphism $x: H \to W$ that we interpret as an \textbf{experience relation}.

This morphism captures the experience of a possible world, e.g., seeing the whiteness of some snow. More precisely, it maps a human uncertainty state $h \in H$ to the set $x(h) \subseteq W$ of all possible worlds compatible with that state. For example, if $h$ represents ``uncertainty about the colour of snow'', $x(h)$ might contain multiple worlds, e.g., $\{w_{\text{white}}, w_{\text{gray}}\}$.

The human and LLM routes are processes that attempt to \textit{resolve} this initial uncertainty by consulting propositions. Recall the two distinct levels of successful resolution:

\begin{itemize}[leftmargin=*]
    \item \textbf{Sound Agreement ($H^{\ast}$):} This is the set in which the LLM path is \textit{sound} (non-hallucinatory) relative to the human path. It allows the response to be ambiguous (since $P_{\mathrm{human}}(h)$ might include multiple worlds) or incomplete.
    \item \textbf{Unique Agreement ($H^{!}$):} This is the most stringent set in which both paths successfully resolve the initial uncertainty down to the same unique world.
\end{itemize}

\noindent Then we have:

\begin{proposition}
If $h^{\ast} \in H^{!}$ is such that the experience maps to the actual world (i.e., $x(h^{\ast}) = \{w^{\ast}\}$), then the output $\bar{e} = (e \circ i_{g_0} \circ p)(h^{\ast})$ constitutes \emph{knowledge}.
\end{proposition}

\begin{proof}
    Trivial. If $w^{\ast}$ is the actual world, and the human experience resolves the uncertainty in the unique agreement situation $h^{\ast} \in H^{!}$ (meaning $W_{\mathrm{human}}(h^{\ast}) = W_{\mathrm{AI}}(h^{\ast}) = \{w^{\ast}\}$), then the interpretation of the output yields the actual world: $(\rho \circ r)(\bar{e}) = w^{\ast}$. Thus, $\bar{e} \in O$ satisfies the condition of being a truthful content warranted by both experience ($x$) and the reliable alignment of the process, which corresponds to knowledge.
\end{proof}

\section{Interpretation of the categorical framework}
\label{sec: categorical_analysis}

\noindent In this section, we provide a detailed interpretation of the components of category $\mathcal{C}$, endowing these objects within $\mathbf{Rel}$ with clear epistemic features.

\paragraph{}
\begin{itemize}
\item \emph{$H$ (human epistemic situations)}. This component represents the state of a human user situated within a context of inquiry. A human epistemic situation encompasses the person's background knowledge and information, their specific epistemic need or query, and the broader environmental or task context within which the inquiry arises. Essentially, $H$ captures the theoretical notion of a user having a question and actively seeking information to resolve an epistemic gap. By explicitly marking the human's situational context, we acknowledge that a user's cognitive state---including what they already know or are at least informed about, and what they are trying to know or are at least seeking to be informed about---constitutes an integral component of the epistemic process, rather than an external factor. In cognitive science terminology, this corresponds to the \textit{problem context} that an agent brings to bear when approaching a knowledge source \parencite{newell_knowledge_1982}. This conceptualisation also aligns with the notion of how external tools and informational resources become constitutive elements of a `cognitive situation' for an agent, as advanced in the extended mind thesis \parencite{clark_extended_1998}.

\item \emph{$C$ (content repository)}. This component represents the corpus of human-authored content that serves as an informational resource. This includes books, scholarly articles, websites, encyclopedias, research papers, and other sources created by human agents that one might consult to answer epistemic queries. $C$ can be conceptualised as representing an extensive library of all available content on the Internet. This component emphasizes that human knowledge is characteristically recorded externally in linguistic form, constituting `environmental information' \parencite{floridi_philosophy_2011}. In particular, the content \textit{C} consists of symbols that have been endowed with syntax and semantics by human authors \parencite{floridi_philosophy_2011}. 

In Shannon and Weaver's information theory \parencite{shannon_mathematical_1949}, semantics is explicitly ignored because it is irrelevant to the engineering problem of data transfer. In this article, we operate at a LoA where semantics is paramount. Some recent work also supports this view. For example, 
\textcite{bai_forget_2025} proposes a semantic information theory that focuses on the Level-B semantic problem—``How precisely do the transmitted symbols convey the desired meaning?''—and argues that the \textbf{tokens} should be regarded as the fundamental unit of analysis. Our concern is with truth-bearing propositions, which both human agents and artificial systems aim to produce. We therefore define $C$ as the repository of syntactically well-formed and semantically meaningful content. The content in $C$ itself may contain untruths,\footnote{We do not use the term `false' or `falsehood' to avoid any discussion about \textit{bivalence}, which would be irrelevant here. Readers who have no problems with bivalence can read `untruths' as equivalent to `falsehoods'.} reflecting the mixed quality of any real-world corpus. 

In this context, we define $g$ as a relation, i.e., a morphism in $\mathcal{C}$. It maps a content $c \in C$ to the set of all its plausible propositions $g(c) \subseteq \mathrm{Pred}(W)$. This set may include various (and possibly contradictory or untruthful) interpretations, realistically reflecting the mixed quality of the content in $C$ itself.
The benchmark for success is clarified by the output set of the entire composite human path for a given $h$, which we denote $P_{\mathrm{human}}(h) = (g \circ c)(h)$. This set represents the ground-truth against which the LLM's output set, $P_{\mathrm{AI}}(h)$, is measured. This measurement uses the success criterion of entailment ($P_{\mathrm{AI}}(h) \subseteq P_{\mathrm{human}}(h)$). We shall return to this methodological choice in Section~\ref{sec:limitations}.

\item \emph{$C^{\prime}$ (tokenised content/prompts)}. This represents content converted into token sequences (typically subwords or character sequences) that LLMs can process. Unlike humans who comprehend natural language directly, LLMs operate exclusively on numerical token identifiers. $C^{\prime}$ encompasses both training data and user prompts in machine-readable format. Similarly, all content comprising the LLM's training data originally consisted of human-readable content that was tokenised in $C^{\prime}$ format. By including $C^{\prime}$ as an explicit analytical stage, we emphasise a fundamental difference in how human and artificial agents access content: humans read linguistic expressions for their semantic content, whereas LLMs process numerical token identifiers. This distinction can obscure important nuances; for instance, some information may be lost or distorted during the tokenisation process, constituting one potential source of the mismatches we will analyse. At the LoA used in our diagram, we include $C^{\prime}$ to acknowledge explicitly that a transformation occurs from content to machine-readable format.

\item \emph{$D(C^{\prime})$ (training datasets constructed from tokenised content)}. 
This represents structured training data: curated subsets of tokenised content from journals, books, websites, and other sources. At the chosen LoA, $D(C^{\prime})$ abstracts away many details concerning how the dataset is prepared---including sampling methodologies, cleaning procedures, or preprocessing steps---emphasising that LLM training material derives entirely from human-generated content. This component serves as a reminder that an LLM's apparent information ultimately originates from human textual production, albeit in computationally processed form. This observation aligns with the critical insight that contemporary LLMs do not possess independent information in isolation but are epistemically parasitic (in a technically neutral sense, see \parencite{harnad_language_2024}) upon content that humans have created. As Bender et al.\ memorably characterise it \parencite{bender_dangers_2021}, LLMs function as `stochastic parrots', synthesising patterns from their training data without any understanding. Within this framework, $D(C^{\prime})$ represents the parrot's `feeding corpus', so to speak.

\item \emph{$G$ (space of trained LLMs)}. This component represents the mathematical space of all possible LLM states that could arise during training. One can conceptualize $G$ as an abstract parameter space in which each point $g \in G$ corresponds to a fully trained LLM characterised by particular learnt parameters. When we `train an LLM', we are essentially selecting one specific point from this space. For example, each member of the series GPT-\textit{x} corresponds to a particular $g_0$ in this space $G$, representing the specific set of weights obtained after training on a designated dataset. By treating $G$ as a formal object in the diagram, we emphasise that the LLM (once trained) constitutes an artefact constructed from data processing. This conceptualisation also acknowledges the notion in machine learning theory of a hypothesis space or parameter space. Importantly, $G$ is conceptually distinct from the dataset $D(C^{\prime})$; $G$ contains the results of the training process (the trained LLMs), while $D(C')$ includes the data itself. In more philosophical terms, if $C$ represents the realm of explicit content (documents and texts), then $G$ represents the realm of implicit content or learnt `dispositions', encoded within the LLM's weight parameters. This distinction parallels the classical dichotomy in epistemology between explicit propositional knowledge (knowing that) and embodied procedural knowledge (knowing how), although an LLM's `know-how' consists of statistical correlations rather than any understanding \parencite{floridi_ai_2023}. Throughout our analysis, we will focus on a single, trained LLM $g_0 \in G$, for instance, a fixed instantiation of the GPT-x series that we use to answer queries.

\item \emph{$O$ (outputs)}. $O$ represents the outputs generated by the LLM in response to a given prompt. This typically consists of a sequence of tokens that can be detokenised into human-readable text. For instance, if the prompt is `What is the capital of France?', the LLM's output might consist of the token sequence corresponding to `Paris'. We include $O$ to mark the precise point at which the LLM produces something in response to an input. It is crucial to note that $O$ is formally distinct from $\mathrm{Pred}(W)$, a point we will address subsequently. $O$ represents merely the data (string or artefact) that the machine produces, and which require interpretation to become meaningful. The systematic separation of $O$ from $\mathrm{Pred}(W)$ reflects the fundamental insight that a sentence does not yet constitute a statement about \textit{W} until its semantic content is determined. A generated output such as `Paris' constitutes, in itself, merely a string of five letters. We must map it to its propositional content (`Paris is the capital of France') to assess its truth value. This separation proves crucial for avoiding the conceptual confusion between strings of symbols (linguistic expression) and meanings, a confusion that lies at the heart of why an LLM can appear knowledgeable through the production of plausible text without actually `knowing' in the human sense \parencite{searle_minds_1980, floridi_ai_2023}.

\item \emph{$W$ (possible worlds)}. $W$ formally represents a state space of relevant, possible worlds (including the actual one) that constitute the subject matter of our discourse. It encompasses what statements are about. For example, $W$ includes facts such as the geographical and political relationship between Paris and France, Canberra and Australia, Coruscant and the Old Republic, or any other states of affairs to which claims can or fail to \textit{correspond}.\footnote{For the sake of simplicity, we are assuming a correspondence theory of truth à la Tarski, but in other contexts, one of the authors (LF) has defended a ``correctness'' approach (\parencite{floridi_semantic_2011}).} In abstract terms, $W$ functions as the truth-maker for propositions: it constitutes the domain that determines what truth value qualifies any given claim. Philosophers routinely invoke `the world' in discussions of truth and reference. Here, we include $W$ to provide an explicit component in the diagram that serves as the ultimate referent for all our interpretations. This helps maintain theoretical clarity about what we ultimately care about: that the information  provided (whether by content or by the LLM) accurately reflects something about the state of affairs in which we are interested.

\item \emph{$\mathrm{Pred}(W)$ (propositions about $W$).} This component is essential because it formally represents the claims or propositions about $W$ that can be formulated. $\mathrm{Pred}(W)$ can be conceptualised as the space of propositions that possess a truth-value with respect to $W$. One can think of $\mathrm{Pred}(W)$ as analogous to the power set of $W$ or some logical space of propositions. For our analytical purposes, it constitutes the target space for both content and LLM outputs. We map things to $\mathrm{Pred}(W)$ rather than directly to $W$ because we want to maintain the crucial distinction between talking about a possible world and the possible world itself. A sentence such as `Paris is the capital of France' does not literally constitute the city of Paris or the country of France; it represents a statement about that state of affairs. 

By employing $\mathrm{Pred}(W)$, we ensure that the diagram compares like with like: for any human epistemic situation $h \in H$, the human route yields a set of propositions ($P_{\mathrm{human}}(h)$), and the LLM route yields a set of propositions ($P_{\mathrm{AI}}(h)$). These sets can then be formally compared. This reflects a methodological concern about maintaining logical precision. It aligns with the long-established tradition in logic and philosophy of language that emphasizes the fundamental importance of semantics as the systematic mapping of language to propositional content \parencite{montague_universal_1970,dretske_knowledge_1981}. By operating at the level of propositions, we avoid conceptual confusion between syntax and semantics, as well as disquotational confusions \parencite{tarski_semantic_1944,quine_philosophy_1970}, a point to which we will return in Section~\ref{sec:lessons}. In other words, $\mathrm{Pred}(W)$ constitutes the domain in which meaning resides, while $W$ represents the domain in which truth is evaluated. This explicit separation in the diagram enables us to define ``success'' formally---such as the soundness: $P_{\mathrm{AI}}(h) \subseteq P_{\mathrm{human}}(h)$---by comparing these output sets within $\mathrm{Pred}(W)$.

\end{itemize}

\noindent Having systematically described the components $H$, $C$, $C'$, $D(C')$, $G$, $O$, $W$, and $\mathrm{Pred}(W)$, we now have the complete list of objects of $\mathcal{C}$, which are the theoretical entities for our framework: a human agent in an epistemic context ($H$) seeking information (or even knowledge) about some possible worlds ($W$), a repository of content ($C$) which can be tokenised (to $C^{\prime}$) and used to train a LLM (via $D(C^{\prime})$ into some $g_0 \in G$), such that the LLM's outputs ($O$) in response to prompts correspond (or fail to correspond) to meaningful propositions ($\mathrm{Pred}(W)$) about the same $W$. 

The structure of $\mathcal{C}$ is summarised in the diagram in Figure~\ref{fig:1}, which depicts a comparison between two distinct routes from a human epistemic situation ($H$) to the space of propositions about the world ($\mathrm{Pred}(W)$). One route represents the set of propositions that a human agent can derive using external informational resources. The alternative route captures the set of propositions that an LLM produces. Through side-by-side examination of these pathways, we can precisely articulate what it means for them to agree (defined as entailment), i.e., that the LLM's output set is a subset of the human path's set, and systematically analyse why they diverge/produce hallucinations. More precisely:

\begin{itemize}
\item The {\bf human epistemic path} is the upper pathway in the diagram. Starting from a human epistemic situation $h \in H$ (representing a person with a query or informational need), 
this composite relation maps $h$ to an output set of propositions $P_{\mathrm{human}}(h)$, generated by first applying the ``consult content'' relation to find all relevant content, and then applying the ``interpret'' relation to find all plausible propositions from that content. This output set $P_{\mathrm{human}}(h)$ encapsulates the ground-truth or benchmark set of propositions against which the AI path is measured.
\item The {\bf LLM path} is the lower pathway in the diagram. This route captures how an LLM produces an output set of propositions given a human epistemic situation $h \in H$. The initial step is the ``prompt'' relation, which maps the question $h$ to a set of possible tokenised prompts (e.g., different phrasings of the same question). 
Subsequently, a function pairs each prompt with the fixed LLM. The stochastic ``evaluation'' relation maps that to a set of possible output token sequences. Finally, the ``interpret'' relation maps each such token sequence $o$ to a set of plausible propositions, capturing semantic ambiguity.
For example, for the query $h$ = `What is the capital of France?', this output set $P_{\mathrm{AI}}(h)$ would ideally contain the proposition `Paris is the capital of France'.
\end{itemize}

\noindent Figure~\ref{fig:2} shows the {\bf offline training pipeline} of the LLMs. This pipeline captures the uncertainty and non-determinism of the process. $s$ is the tokenisation relation. $D$ constructs a tokenised corpus, yielding a set of possible datasets (reflecting different sampling or cleaning choices). $t$ describes the training process. This is a \textbf{stochastic relation}: due to factors like random weight initialisation and data shuffling, a dataset $d$ can be mapped to a set of possible trained models $\{g_1, g_2, \dots\} \subseteq G$. A fixed $g_0$ can be viewed as the result of only one instance of this one-to-many relation.

This pipeline emphasises a crucial epistemological fact: an LLM’s `knowledge' comes only from what is represented in its training data \parencite{marcus_road_2022}. 
The fact that each stage can introduce distortions is the radical reason why these morphisms must be modelled as relations, not as simple, idealised functions.

As discussed, we define ``success'' as the \textbf{entailment} relation, requiring that the AI path's output must be contained in the human path's ground-truth set. We defined $H^{\ast} \subseteq H$ as the subset of epistemic situations where the entailment criterion holds. This set represents the domain where the AI path is \textbf{sound/ non-hallucinatory} relative to the human path. This constitutes the LLM's domain of epistemic soundness, namely, the set of queries for which the AI system functions as a reliable (not necessarily complete) information source, in the sense that any proposition it yields is guaranteed to be contained within the human ground-truth set. If $H^{\ast} = H$, the LLM achieves complete \textbf{soundness} (reliable alignment) with human epistemic processes.  In practice, $H^{\ast}$ is typically large but not total, and alignment research fundamentally aims to expand it systematically.

Our central thesis is that ``hallucinations'' are not mere implementation bugs, but are intrinsic structural failures of the LLM architecture, or more specifically, failures of the artificial path to align with the human path. To formalise this, we require a framework that can unify all the distinct failure modes (tokenisation ambiguity, dataset bias, stochastic inference, interpretive ambiguity, etc.) into a \textbf{single, coherent definition of failure}. The subcategory $\mathcal{C}$ of $\mathbf{Rel}$ is the ideal Level of Abstraction (LoA) for this task because it abstracts away the \textit{sources} of non-determinism (whether it is possible ambiguity or probabilistic stochasticity) and flattens them all into a single \textbf{relation}. Hallucination is  simply defined as any $h$ for which an inclusion relation fails.

\section{Extensions and Alternative Conditions}
\label{sec:refinements}

Our suggested framework, founded on the category $\mathbf{Rel}$ and using entailment-commutativity (defined based on $\subseteq$ instead of $=$), is not the only possible categorical analysis, it can be extended or specialised to serve as a foundation for at least two deeper, more specialised investigations:

	%\begin{itemize}[leftmargin=2em]
		%\item A \emph{sheaf perspective} over \(H\), where \(P_{\mathrm{human}}\) and \(P_{\mathrm{AI}}\) would be seen as sections over epistemic situations, allowing one to ask whether soundness defines an open set of the topology of $H$ (\parencite{rosiak2022sheaf}).
		%\item A \emph{probabilistic refinement} using Markov categories such as \(\mathsf{Stoch}\) or categories of stochastic relations such as \(\mathsf{SRel}\), or using probability monads, to distinguish genuinely stochastic morphisms (e.g.\ sampling) from merely ambiguous ones.
	%\end{itemize}

\begin{enumerate}
    \item \textbf{A Topological (Sheaf) Perspective.} 
    We could re-characterise our relations $P_{\mathrm{human}}$ and $P_{\mathrm{AI}}$ as \textbf{sections} of a bundle over the base space $H$ (i.e., $sec: H \to \mathcal{P}(\mathrm{Pred}(W))$). The benefit of this (pre)sheaf perspective is that it allows us to ask deeper \textit{topological} questions about the nature of alignment. For example: \textit{Is ``soundness'' a stable property (an ``open set'' in $H$)}? Or \textit{is it, as empirical evidence suggests, ``brittle'' (a subset that is not open)}? \parencite{rosiak_sheaf_2022}.

    \item \textbf{A Probabilistic (Hybrid) Perspective.} 
    Since $\mathbf{Rel}$ abstracts away from probabilities, it flattens stochastic processes into ``support'' relations (i.e., focussing on the sets of entities with positive probability). We could ``un-flatten'' our relations to distinguish their origins. A more rigorous analysis would use a \textbf{hybrid framework} employing \textbf{Markov categories} (e.g., $\mathbf{Stoch}$, see \parencite{shiebler_categorical_2021}), or stochastically enriched categories such as $\mathbf{SRel}$, see \parencite{brown_categories_2009}) for the stochastic morphisms ($e, s$) while keeping $\mathbf{Rel}$ for the ambiguous ones ($p, g, r$). This would allow the outputs to be probability distributions over $O$, inducing distributions over $\mathrm{Pred}(W)$. Alignment could then be measured probabilistically (e.g., via KL divergence or other metrics between the human and AI output distributions) rather than by the set inclusion of their supports. The benefit of this approach is that it would allow us to \textit{disentangle and quantify} the sources of failure (i.e., to measure how much of a hallucination is due to stochastic sampling vs. interpretive ambiguity).

\item \textbf{Partiality and Epistemic Humility (A Feature of $\mathbf{Rel}$)}. 
The case where an LLM appropriately refuses to answer (``epistemic humility'') is not a special refinement, but is naturally handled by our $\mathbf{Rel}$ framework. A ``refusal to answer'' for a given $h$ is modelled as the AI path yielding the \textbf{empty set} of propositions:
\[
P_{\mathrm{AI}}(h) = (r \circ e \circ i_{g_0} \circ p)(h) = \varnothing.
\]

We can now analyse this case using the soundness criterion ($H^{\ast}$):

\[
P_{\mathrm{AI}}(h) \subseteq P_{\mathrm{human}}(h)
\quad \Longleftrightarrow \quad
\varnothing \subseteq P_{\mathrm{human}}(h)
\]
Since the empty set ($\varnothing$) is a subset of \textit{every} set, this condition is \textbf{always true}. This reveals a powerful feature of our model: the framework correctly classifies ``epistemic humility'' (a refusal) as sound (i.e., $h \in H^{\ast}$). A refusal to answer is the ultimate form of ``non-hallucination.

\end{enumerate}

\noindent However, for the central thesis of \textit{this} paper, which is to identify the \textit{unified structure} of architectural failure and the circumvention of the symbol grounding problem, the category $\mathbf{Rel}$ suffices as the most appropriate choice.

%\begin{enumerate}
  %  \item \textbf{A Topological (Sheaf) Perspective:} We could re-characterise our relations $P_{\mathrm{human}}$ and $P_{\mathrm{AI}}$ as \textbf{sections} of a bundle over the base space $H$ (i.e., $sec: H \to \mathcal{P}(\mathrm{Pred}(W))$). The benefit of this (pre)sheaf perspective is that it allows us to ask deeper \textit{topological} questions about the nature of alignment. For example: \textit{Is ``soundness'' a stable property (an "open set" in $H$)}? Or \textit{is it, as empirical evidence suggests, ``brittle'' (a subset that is not open)}?

 %   \item \textbf{A Probabilistic (Hybrid) Perspective:} We could ``un-flatten'' our relations to distinguish their origins. A more rigorous analysis would use a \textbf{hybrid framework} employing \textbf{Markov categories} (e.g., $\mathbf{Stoch}$, see \parencite{shiebler_categorical_2021}, or $\mathbf{SRel}$, see \parencite{brown_categories_2009}) for the stochastic morphisms ($e, s$) while keeping $\mathbf{Rel}$ for the ambiguous ones ($p, g, r$). The benefit of this approach is that it would allow us to \textit{disentangle and quantify} the sources of failure (i.e., to measure how much of a hallucination is due to stochastic sampling vs. interpretive ambiguity).
%\end{enumerate}

\section{Sources of Mismatch: Analysing Systematic Failure Points}
\label{sec:sources of mismatch}

When the routes diverge, producing $h \notin H^{\ast}$, the failure can be traced to one or more stages in the processing pipeline.

\begin{itemize}

\item \emph{Tokenisation distortion ($s$).} Information may be lost or systematically distorted during the conversion from $C$ to $C'$.

\item \emph{Dataset construction bias ($D$).} Particular information can be systematically excluded or represented disproportionately (bias).

\item \emph{Training generalisation failure ($t$).} The learning process may develop flawed generalisations.

\item \emph{Prompting ambiguity ($p$).} Poorly formulated or ambiguous questions yield unclear tokenised inputs $C'$.

\item \emph{Inference stochasticity ($e$).} Stochastic generation may produce incorrect tokens even when the LLM has implicitly `learnt' the correct factual information.

\item \emph{Interpretation failure ($r$).} Ambiguity, brevity, or context-dependence of outputs may mislead the human interpreter attempting to extract propositional content.

\end{itemize}

\noindent Each problem, alone or in combination with the others, represents a potential failure point where human and LLM routes systematically diverge, producing distinct and potentially contradictory propositions about the relevant referent $W$. We shall return to this point in Section~\ref{subsection non alignments}.

\section{An Illustrative Example}
\label{sec:example}
Consider Alice, who asks: ``What are the primary inputs for photosynthesis?''. Her query corresponds to a human epistemic situation $h \in H$.

\paragraph{Human Epistemic Path.}
The human route $(g \circ c)$ maps $h$ to the ground-truth set of propositions $P_{\mathrm{human}}(h)$. In this case, the set represents the scientifically accepted primary inputs for photosynthesis. (We use simple symbolic labels for the abstract propositions):
\[
P_{\mathrm{human}}(h) = \{ p_{\text{Light}},\ p_{\text{CO}_2},\ p_{\text{Water}} \}.
\]

\paragraph{Artificial Computational Path (Successful Case).}
The LLM ($e$) generates an output $o \in O$ corresponding to ``Light and $\mathrm{CO}_2$''.
Through interpretation ($r$), this output is mapped to the following set of propositions:
\[
P_{\mathrm{AI}}(h) = \{ p_{\text{Light}},\ p_{\text{CO}_2} \}.
\]

\noindent The inclusion $P_{\mathrm{AI}}(h) \subseteq P_{\mathrm{human}}(h)$ holds:
$\{ p_{\text{Light}}, p_{\text{CO}_2} \} \subseteq \{ p_{\text{Light}}, p_{\text{CO}_2}, p_{\text{Water}} \}$ is \textbf{true}.
Therefore, $h \in H^{\ast}$.
Our framework correctly classifies this AI response as \textbf{sound}. It does not introduce hallucinations, even though it is \textbf{incomplete} because it omits $p_{\text{Water}}$. A framework based on strict equality ($=$) would have incorrectly labelled this as a failure.

\paragraph{Counterexample: Artificial Computational Path (Failure Case).}
In contrast, suppose that the LLM ($e$) produces an output $o \in O$ corresponding to ``Light, $\mathrm{CO}_2$, and Soil.''
After interpretation ($r$), the resulting proposition set is:
\[
P_{\mathrm{AI}}(h) = \{ p_{\text{Light}},\ p_{\text{CO}_2},\ p_{\text{Soil}} \}.
\]

\noindent Here, the inclusion $P_{\mathrm{AI}}(h) \subseteq P_{\mathrm{human}}(h)$ \textbf{fails}, because $p_{\text{Soil}} \in P_{\mathrm{AI}}(h)$ but $p_{\text{Soil}} \notin P_{\mathrm{human}}(h)$.
Hence, $h \notin H^{\ast}$, indicating a failure of \textbf{soundness}.
This constitutes a clear instance of \textbf{hallucination}: the AI path has produced a proposition ($p_{\text{Soil}}$) that does not belong to the human ground-truth set \parencite{borji_categorical_2023,openai_gpt-4_2023}.

\section{Some philosophical lessons}
\label{sec:lessons}
The formal analysis presented in the previous sections can help clarify some longstanding debates about LLMs, semantics, and the nature of machine-generated content. There are also several philosophical lessons worth exploring. In the rest of this section, we discuss two that seem more relevant in this article. One is more \textit{metatheoretical} and concerns the value of a categorical analysis. The other is more \textit{epistemological} and concerns how to interpret different failures of entailment as ways in which LLMs can be wrong. We leave a third, which is \textit{semantic} and more straightforwardly philosophical, to the next section because it concerns the symbol grounding problem, the main topic of the article.

\subsection{The Value of a Categorical Analysis}
The first lesson can be introduced as an objection. One may argue that category theory introduces notation without providing any fundamentally new insight. Perhaps, we could have used a simpler flow chart or just set theory with relations, without the full formal machinery of the $\mathbf{Rel}$ category. The category approach wraps the analysis in concepts such as ``poset structures'' and ``composite relations'' which, no matter how relevant, may raise the question whether they yield new \textit{results} or just rephrase things more formally, like the entailment criterion $P_{\mathrm{AI}}(h) \subseteq P_{\mathrm{human}}(h)$. Even if leveraging categorical thinking helps consider universal properties, it does not prove anything, but rather corroborates intuitive ideas of alignment/soundness. The response, and hence the lesson learnt, is twofold.

On the one hand, the conceptual clarity and expressive power of the formalism we provided should help in contexts where anthropomorphic language, cultural and conceptual prejudices, and pre-theoretical assumptions significantly impact understanding. The debate about LLMs' capabilities depends partly on the ambiguity of loaded terms such as  `understanding', `learning', `thinking', `reasoning', `hallucinations', `knowledge', or `intelligence' \parencite{floridi_anthropomorphising_2024}.
Following the classical analytic view (Frege, Russell, Carnap), good conceptual design is not mere paraphrase but works as \textbf{a} methodological tool that helps resolve problems and terminological disputes by exposing their logical form. Indeed, later traditions---from Quine's regimentation to Lewis's \textbf{analyses} and pragmatist explication---treat clarification as directly problem-solving. In our case, separating and formalising the human epistemic paths from the LLM computational path enhances understanding, ultimately making the debate more precise and fruitful, even if participants only agree to disagree. We present our analysis as a neutral map, allowing different philosophical perspectives to be located based on the elements they emphasise and how they interpret the conditions for agreement (e.g., our entailment criterion $H^{\ast}$, or the stricter equality/commutativity).

On the other hand, by casting the problem categorically, our analysis \textbf{demonstrates its extensibility}. We have methodologically employed the category of relations $\mathbf{Rel}$ to model ambiguity and non-determinism. This framework immediately connects to further extensions needed to capture the full probabilistic nature of `$e$' and `$s$'. We saw that these include \textbf{Markov categories} like $\mathbf{Stoch}$ (which models probability measures) or categories of \textbf{stochastic relations} like $\mathbf{SRel}$ (which can also model sub-probability measures and thus partiality). An equivalent approach would be to use \textbf{probability monads} to construct the relevant Kleisli categories. In fact, other researchers have begun to use category theory to analyse LLM reasoning and limitations (\parencite{jia_category-theoretical_2024}, see also the studies we refer to in Section 8). In particular, \textcite{mahadevan_rose_2025} has addressed the challenge of semantic equivalence---how an LLM should \textbf{recognise} that `Charles Darwin wrote' and `Charles Darwin is the author of' mean the same thing---using categorical homotopy theory. This approach models semantic paraphrases as `weak equivalences' in an `LLM Markov category', aiming to ensure that semantically identical inputs yield homotopically equivalent probabilistic outputs. This resonates strongly with our own methodological move to abandon strict equality ($=$) in \textbf{favour} of a weaker, more practical relation (our entailment-commutativity, defined in terms of $\subseteq$). Our framework thus contributes to a shared, formal enterprise aimed at moving the analysis of LLMs beyond metaphor and towards mathematical \textbf{rigour}.

\subsection{Non-alignments (Entailment Failure) as ways in which LLMs can be wrong}
\label{subsection non alignments}
The second lesson concerns LLMs' hallucinations. We saw that \textit{if} our entailment criterion holds ($h \in H^{\ast}$), the entire LLM's output set is contained within the set of propositions that the human would have gotten from consulting reliable content $C$. This is a strong form of soundness (reliability): the LLM is expected to produce \textit{only} information that is supported by the human path $(g \circ c)$. In this case, the LLM has effectively arrived at a correct, well-grounded proposition, thanks to training on human content. However, sometimes, the entailment criterion will fail ($h \notin H^{\ast}$). This occurs when $P_{\mathrm{AI}}(h) \not\subseteq P_{\mathrm{human}}(h)$, which is our formal definition of a structural failure, or ``hallucination'' in a broad sense.
\[
\text{Failure}(h) \;\Longleftrightarrow\; \exists p \text{ such that } p \in P_{\mathrm{AI}}(h) \text{ and } p \notin P_{\mathrm{human}}(h).
\]

\noindent The power of the proposed $\mathcal{C}$ framework is that this definition is able to capture \textit{all} types of failure, including those identified in recent incident \textbf{analyses} (cf. \parencite{maliugina_llm_2024}). Consider these real-world examples:

\begin{itemize}
    \item \textbf{Case 1: Factual Hallucination (e.g., Air Canada).}
    A user asks the Air Canada chatbot for its bereavement refund policy ($h$).\footnote{\href{https://decisions.civilresolutionbc.ca/crt/crtd/en/item/525448/index.do}{Moffatt v. Air Canada}, 2024 BCCRT 149. See also Kyle Melnick, ``Air Canada chatbot promised a discount. Now the airline has to pay it.,'' Washington Post, February 18, 2024; Maria Yagoda, ``Airline held liable for its chatbot giving passenger bad advice – what this means for travellers,'' BBC, February 23, 2024.} The human path, $P_{\mathrm{human}}(h)$, is the set of propositions derived from the \textit{actual} company policy documents ($C$), which state the restrictive policy:
    \[ P_{\mathrm{human}}(h) = \{ p_{\text{refund-after-flight}} \} \]
 However, the LLM path hallucinates and invents a non-existent policy. This failure can manifest itself in (at least) two ways:
    \begin{itemize}
        \item \textbf{Replacement (flat-out wrong/contradiction).} The AI invents a policy \textit{instead} of the correct one: \[ P_{\mathrm{AI-a}}(h) = \{ p_{\text{refund-before-flight}} \} \]
        \item \textbf{Contamination (logically weaker).} The AI ``helpfully'' provides \textit{both} the correct policy and the invented one:
        \[ P_{\mathrm{AI-b}}(h) = \{ p_{\text{refund-after-flight}}, p_{\text{refund-before-flight}} \} \]
    \end{itemize}
The defined soundness criterion $P_{\mathrm{AI}}(h) \subseteq P_{\mathrm{human}}(h)$ correctly identifies both scenarios as failures ($h \notin H^{\ast}$). This demonstrates the robustness of our framework: it correctly classifies any introduction of non-ground-truth propositions as a failure.

\item \textbf{Case 2: Contextual Failure (e.g., Klarna).}
    A user asks the Klarna support chatbot how to write Python code ($h'$).
    The human path, $P_{\mathrm{human}}(h')$, is derived from \textit{Klarna's knowledge base} ($C$), which only contains customer service information.
    The ground-truth set for this ``off-topic'' query is arguably the empty set, or a refusal:
    \[ P_{\mathrm{human}}(h') = \{ p_{\text{I-only-handle-Klarna-services}} \} \]
    However, the LLM path ignores its context ($C$) and answers using its general memorised knowledge:
    \[ P_{\mathrm{AI}}(h') = \{ p_{\text{use-list-comprehension...}} \} \]
    This is also a failure ($h' \notin H^{\ast}$), because $P_{\mathrm{AI}}(h') \not\subseteq P_{\mathrm{human}}(h')$.
    The LLM produced a proposition that is factually correct but not supported by its designated content corpus.\footnote{See https://www.evidentlyai.com/blog/llm-hallucination-examples.}
\end{itemize}

\noindent All these cases highlight the strength of using $\mathbf{Rel}$ and \textbf{entailment ($\subseteq$)}. Our framework unifies seemingly different failure types. Both \textit{factual hallucinations} and \textit{contextual failures} or \textit{unintended \textbf{behaviors}} are revealed to be the \textit{exact same structural failure}: $P_{\mathrm{AI}}(h) \not\subseteq P_{\mathrm{human}}(h)$. This supports our central thesis that these are not isolated bugs, but intrinsic failures of the architecture to remain entailed by its grounding corpus $C$.

Finally, recent surveys (\parencite{ji_survey_2023}, \parencite{rawte_survey_2023}, \parencite{qi_survey_2024}) systematise the phenomenon of hallucination, underscoring that these \textbf{``entailment failures''} have multiple, measurable subtypes, with distinct detection and mitigation strategies. This supports our claim that hallucinations are structural, not incidental.

\section{The Symbol Grounding Problem: Circumvented, Not Solved}
\label{sec:SGP}

The third and main lesson concerns the symbol grounding problem, as classically formulated by \textcite{harnad_symbol_1990}: how systems that process symbols (including tokens) based on syntactic rules (including logical and statistical) can establish meaningful connections between those symbols and their corresponding real-world referents. Due to the recent success of LLMs, this fundamental question has regained importance, as these systems exhibit impressive linguistic performance despite being trained solely on text \parencite{pavlick_symbols_2023}. In this section, we argue that LLMs do not solve the symbol grounding problem. Instead, they circumvent the need for it through a sort of \textit{epistemic parasitism} (\cite{harnad_language_2024}): they operate exclusively on a corpus of content $C$ that human agents have already grounded through embodied experience, causal interaction with possible worlds $W$, and participation in socio-cultural practices about them.

However, before developing our argument, it is important to clarify that `grounding' in this discussion encompasses at least three interrelated but distinct aspects: (i) perceptual grounding through sensorimotor coupling with the world, (ii) causal-informational grounding through reliable information channels from the world to representation, and (iii) normative-social grounding through participation in reason-giving practices within linguistic communities. Although these aspects are mutually reinforcing in human cognition, they can be analytically distinguished, and our argument addresses all three dimensions.

In our analysis, the human path includes an implicit but crucial solution to the symbol grounding problem. The content repository $C$ consists of texts written by human authors whose linguistic symbols are already grounded. This grounding occurs through multiple and mutually reinforcing mechanisms that LLMs inherently lack; any grounding available is human-delivered.

First, human agents have a direct causal-perceptual coupling with the world. We represent this by the \textit{x} arrow in Figure~\ref{fig:1}. When a human writes `snow is white', that statement is typically linked to sensorimotor experience: seeing snow, feeling its coldness, observing its reflective qualities under different lighting conditions, noting how white things are hard to find when they fall on the snow, etc. \parencite{barsalou_perceptual_1999, harnad_language_2024}. The author's neural states bear systematic causal connections to snow itself, not merely to some descriptions of snow. These causal chains, established through perception and action, constitute what \textcite{fodor_psychosemantics_1987} called `informational relations' and what \textcite{dretske_knowledge_1981} analysed as the flow of information from the world to the mind.

Second, human grounding involves sensorimotor engagement with affordances. Understanding `grasp' involves not merely knowing that it co-occurs with `hand' and `object' in text, but having activated motor schemas for grasping, that is, knowing what it feels like to close one's fingers around an object, to adjust grip pressure, to coordinate visual and proprioceptive feedback, or to have been too weak to grasp a heavy object \parencite{gallese_brains_2005}. This embodied dimension of meaning, extensively documented in cognitive science research \parencite{barsalou_perceptual_1999, borghi_action_2014}, provides a grounding that goes beyond distributional patterns in language.

Third, humans ground their symbols as users who participate as normative agents within linguistic communities. Here we follow Wilfrid Sellars's work, particularly \parencite{sellars_empiricism_1956}, where he introduces the concept of the `logical space of reasons', an idea later used by Robert Brandom to develop his inferentialist account \parencite{brandom_making_1994}, according to which to understand a concept is to be able to provide and request reasons, to undertake commitments and recognise entitlements, and to be held accountable by other members of the community. When Alice writes `snow is white', she makes a claim that she can be called upon to justify and for which she bears epistemic responsibility. This normative aspect of language use, the structure of accountability and commitment, is constitutive of meaning in Brandom's view, and it requires genuine agency.

Consequently, the content in $C$ is not just a collection of token sequences; it already holds meaning. Human authors have undertaken the difficult epistemic work of connecting symbols to real-world referents through perception, action, and social practice. The texts in $C$ encode this pre-established semantic content. When we consult Wikipedia to learn that `Paris is the capital of France', we are accessing information created by authors who have causal-historical connections to Paris, whether through direct experience (visiting the city), indirect testimony (reading reliable historical documents), or participation in a community of knowers who collectively uphold such connections. In whatever way the symbol grounding problem is resolved for human language, LLMs do not add to it; instead, they can only subtract from it (through hallucinations). Let us clarify why.

As discussed in Section~\ref{sec: categorical_analysis}, the LLM pathway never includes the experience relation $x: H \to W$, operating entirely within the tokenised content. This architectural fact has profound implications. The entire operation---from training on the dataset $D(C')$ through inference through evaluation $e : G \times C' \to O$---consists of statistical manipulations of symbols that are already meaningful to humans. During training, the LLM learns correlations: that `Paris' co-occurs with `France' and `capital'; that `hot' and `cold' are usually contrasted; that `squeaky clean' appears in contexts describing cleanliness. These are patterns in human language about the world, not patterns in the world itself. The LLM detects second-order regularities---patterns in how humans describe patterns---without accessing the first-order regularities that ground human descriptions.

Consider the contrast explicitly. When a human learns that Paris is the capital of France, the learning process might involve, among other things:

\begin{itemize}
    \item Seeing maps showing Paris's location and political status,
    \item Reading documents produced by other grounded agents,
    \item Visiting Paris and observing government buildings,
    \item Participating in conversations where `capital' is used normatively in political discourse.
\end{itemize}

\noindent Each pathway establishes or strengthens causal-informational connections between the concept `Paris' and the city itself. When an LLM `learns' that Paris is the capital of France, the learning process involves:

\begin{itemize}
    \item Processing millions of token sequences where `Paris,' `capital,' and `France' appear in systematic proximity,
    \item Adjusting weight matrices via gradient descent to predict such sequences more accurately,
    \item Encoding statistical correlations in high-dimensional vector spaces.
\end{itemize}

\noindent The LLM learns that humans say `Paris is the capital of France' in specific contexts, not that Paris is the capital of France. This is not a minor technical distinction, but the difference between genuine reference and sophisticated pattern matching.

This difference can be clarified through the semantic concept of \textit{disquotation}. In Tarski's semantic theory of truth \parencite{tarski_semantic_1944}, the schema ```snow is white' is true if and only if snow is white'' exemplifies the transition from mentioning a sentence to using it to make a claim about reality. This \textit{disquotation} process---moving beyond the linguistic level to assess sentences against real-world states---enables understanding of truth and reference. Humans regularly perform disquotation, using sentences to make claims about the world and evaluating those claims by comparing them (directly or indirectly) to how things are in $W$. When Alice asserts `Paris is the capital of France', she is not merely manipulating a string but making a claim about Parisian geography and French political structure, a claim that is answerable to those worldly facts. LLMs cannot perform genuine disquotation because they operate entirely within a `quoted environment'. They process tokens---strings like `Paris,' `is,' `the,' `capital,' `of,' `France'---but they do not `step outside' those tokens to access Paris or France directly. Their operations remain at the level of syntax and distributional patterns. What the LLM does is better characterised as \textit{re-quotation}: it detects patterns in quoted material (the token sequences in $C'$) and re-presents those patterns in new, statistically plausible and hopefully reliable combinations. It functions within the universe of discourse (human language about the world) but never reaches through that discourse to the world itself.

This is formally shown in our diagram by the fact that the LLM pathway never includes the experience relation $x : H \to W$. The closest it comes to semantics is the relation $r : O \to \mathrm{Pred}(W)$, but crucially, $r$ is an act of interpretation carried out by human users, not by the LLM itself. The human interpreter supplies the semantic bridge, reading the output string `Paris is the capital of France' as an English proposition about Paris and France. The LLM produces syntactic patterns; humans assign meaning.

One might object that this characterisation is too severe. After all, LLMs do exhibit systematic sensitivity to semantic relations. They can answer questions, draw inferences, and show what seems to be conceptual knowledge. \textcite{piantadosi_meaning_2022} advance a version of this objection, arguing that LLMs capture important aspects of meaning through conceptual role semantics. According to this objection, meaning arises from the relationships between internal representational states rather than from direct reference to the world. In this regard, because conceptual role is defined by how internal states relate to each other, the statistical patterns learnt by LLMs might constitute genuine semantic content. Unfortunately, this objection conflates two distinct questions: whether LLMs implement systematic relationships between representations (which they obviously do) vs. whether those relationships constitute the kind of meaning that characterises human cognition (which they do not). It is the latter point, and not the former, that is at stake. Even if we grant that LLMs approximate conceptual role semantics at a formal level, this does not establish that their internal states have the same semantic properties as human concepts, since human conceptual roles are themselves grounded in sensorimotor experience and normative practices. The patterns that LLMs learn are shadows of meaningful human language use, not meaning itself.

Our framework enables us to diagnose precisely why such apparent semantic competence does not amount to genuine grounding, even partially. The main insight is that statistical patterns in language are reflections or proxies of semantic structure, not the structure itself. Since human authors have grounded symbols, and because they use language systematically to communicate about a shared world, the resulting corpus $C$ exhibits statistical regularities that mirror the world structure. Words for causally related entities co-occur; taxonomic relationships are encoded in distributional patterns; inferential connections are reflected in linguistic usage. An LLM trained on $C$ can exploit these regularities to produce outputs that appear grounded, or, one may say, that are grounded but not by the LLM. When it generates `Paris is the capital of France', it has detected that this token sequence is highly probable given specific prompt contexts. This probability reflects the fact that many humans have written this sentence for various reasons, including its truth and usefulness in communication. The LLM's statistical processing enables it to `inherit' the structure that human grounding has imposed on the corpus. But inheritance is not possession. The LLM has access to the fruits of grounding without possessing the capacity for grounding itself. It is like Poe's raven that can repeat `Nevermore' without understanding its meaning---except that the LLM's repetition is statistically sophisticated rather than acoustically faithful.

This analysis connects to a broader debate in cognitive science about the nature of semantic representation. \textcite{mollo_vector_2025} distinguish five types of grounding---sensorimotor, communicative, epistemic, relational, and referential---arguing that LLMs might achieve the first four through their training and deployment contexts, but referential grounding (the capacity to pick out specific worldly entities) remains central to what they call the `Vector Grounding Problem'. They suggest that LLMs might achieve referential grounding through mechanisms like \textit{reinforcement learning from human feedback} (RLHF), which establishes world-involving functions, or even through pre-training alone in limited domains where mechanistic interpretability research reveals world-tracking internal states. However, our categorical framework suggests a more fundamental limitation: even if specific internal states of an LLM track worldly features, this tracking relationship is \textit{parasitic} on the grounding already present in the training corpus $C$. The LLM does not establish its own causal-informational connections to $W$; instead, it detects patterns in the content produced by agents who have established such connections. Thus, any referential relations that are obtained do so \textit{derivatively}, through the mediation of human-grounded content, rather than through the LLM's own engagement with the world.

This is why we describe the LLM's approach as \textit{circumventing} the need for grounding rather than even a partial solution. A partial grounding solution would involve the system developing its own, possibly limited, causal-informational connections to $W$, for example, through experience acquired via sensorimotor interaction within a limited domain---as seen in robotics-integrated systems---or through restricted perceptual modalities, such as multimodal AI systems processing images and text. Indeed, \parencite{carta_grounding_2023} explore one such approach, training LLMs in interactive environments with online reinforcement learning, allowing systems to receive real-time feedback from their actions. Admittedly, such systems show an improved alignment between language and action in limited domains. However, they still face the fundamental limitation that their sensorimotor `experience' consists of processing digital \textit{representations} of sensor data rather than the rich, embodied coupling with the world that characterises human perception and action. Even embodied robots, with richer sensorimotor loops than a thermostat's one-dimensional temperature coupling, still process representations rather than having the kind of experiential access that characterises human perception and action. A vision-language system processes digital image files---themselves human-created representations captured through cameras and encoded in pixels---not the continuous, dynamic perceptual experience that grounds human visual concepts. Similarly, a robot learning object manipulation processes sensor data through learnt statistical regularities, lacking the integrated sensorimotor experience and social-normative context that ground human understanding of affordances and actions \parencite{bisk_experience_2020, incao_roadmap_2025}.
Another study by \textcite{ye_provable_2025} provides strong statistical support from a distinct mathematical perspective for our claim that hallucinations are intrinsic, structural failures rather than bugs. Through experimental analysis, the authors demonstrate that the root cause of such failures lies in the model learning the empirical score function---derived from the limited dataset it has observed---rather than the ground-truth score function, which it ought to learn and that represents the general laws of reality.

\textcite{dove_symbol_2024} offers a complementary perspective on this issue, arguing that the limitations of LLMs reveal something important about human cognition itself. He contends that human semantic memory is only partially grounded in sensorimotor systems and is also dependent on language-specific learning, which he calls `symbol ungrounding'.\footnote{The capacity to reason through language beyond immediate sensorimotor grounding.} According to this view, language provides humans with access to forms of cognition that transcend immediate perceptual experience, allowing us to reason about abstract concepts, distant events, and counterfactual scenarios. LLMs, operating purely on linguistic patterns, demonstrate the richness of language as a source of semantic information. However, as Dove acknowledges, this capacity for `ungrounded cognition' in humans presupposes an embodied foundation that LLMs lack. In our framework, this corresponds to the observation that the success set $H^{\ast}$---where the LLM output aligns with human knowledge---exists only because the training corpus $C$ already encodes the products of human sensorimotor and social grounding. The LLM exploits this grounding without possessing it, borrowing the ungrounded cognitive capacities that language affords humans without having the grounded basis from which those capacities emerge.

As the reader may expect, a similar line of reasoning applies to recent studies suggesting that grounding may be an \textit{emergent property} of scale and multimodality. For example, \textcite{wu_mechanistic_2025} provide causal–mechanistic evidence on the symbol grounding problem. In controlled text and multimodal tests that explicitly separate environmental (\texttt{$\langle$ENV$\rangle$}) from linguistic (\texttt{$\langle$LAN$\rangle$}) tokens, language and vision--language models learn to align words with real--world referents: matched scene or caption cues lower the surprisal of the corresponding word beyond what raw co-occurrence predicts, yielding a positive grounding information gain. Causal interventions then identify the linkage in mid-layer ``aggregate'' attention heads that collect environmental evidence and route it to the prediction position; deliberately disabling part of these heads reliably increases surprisal and degrades performance. However, from our framework's perspective, this does not provide a solution to the symbol grounding problem, but rather a more complex and layered form of circumvention. The `environmental ground' in these studies---be it a pixel array from an image or sensor data from a robot---is still a representation within the content layer $C$ (or its tokenised form $C'$). A VLM correlating the token `Paris' with a cluster of pixels is learning a pattern between two sets of representations, not between a symbol and the experiential interactions with the world $W$. The model has no access to $W$ itself. We saw above that this is corroborated by studies showing the persistent limitations of such emergent grounding. Even emergent mechanisms are best understood as creating highly sophisticated statistical mappings, within the universe of human-generated representations ($C$), which allow the system to mimic more effectively the outputs of a genuinely grounded agent. They refine pattern-matching on pre-grounded data, but they do not bridge the fundamental gap between $C$ and $W$.

As LLMs circumvent the need for grounding by operating exclusively on pre-grounded content, the success set $H^{\ast}$ in our analysis represents precisely those queries where this circumvention strategy works: where the statistical patterns the LLM has learnt happen to match the semantic structure established by human grounding because the training corpus $C$ contains sufficiently consistent, accurate, and well-distributed examples of human-grounded language use about that domain. Outside $H^{\ast}$, the circumvention fails, more or less gently and gradually, not because the LLM's grounding is weak, but because it has no grounding at all, only borrowed patterns that sometimes fail by not borrowing well enough.

The question of whether sufficiently sophisticated sensorimotor coupling could, in principle, constitute genuine grounding at some point in the future remains philosophically contentious. However, current LLMs, multimodal and robotic systems, and other systems similarly designed fall far short of this threshold and operate primarily through pattern recognition over representations rather than through the kind of direct world-engagement required by symbol grounding. Indeed, their success is the source of their limit.

Having argued that LLMs circumvent rather than solve the grounding problem, we can now examine three critical implications of this analysis: the inevitability of hallucinations, the limitations of multimodal architectures, and the inapplicability of use-theoretic semantics.

First, we have already seen that hallucinations occur when the LLM's statistical pattern-completion process generates outputs that, when interpreted by human users (via $r: O \to \mathrm{Pred}(W)$), produce propositions that are untrue or meaningless. In our framework, hallucinations represent failures of entailment. The LLM outputs a string that resembles human language, but it does not correspond to any proposition a well-informed human would assert based on reliable content. The grounding analysis further clarifies why such entailment failures are inevitable rather than merely accidental: they are systematic consequences of the LLM's circumvention strategy. Because the LLM lacks grounding, it cannot distinguish between statistically plausible patterns and genuinely true claims. It might generate `The Eiffel Tower is in Berlin' if its training data happen to contain unusual statistical contexts that make this sequence probable, or if the prompt activates misleading associations in its parameter space. From the LLM's `perspective' (metaphorically speaking), both `Paris' and `Berlin' are merely token identifiers; it has no independent access to Parisian geography to verify its output.

Second, the analysis highlights the limitations of multimodal LLMs. This is an extension of the argument made above concerning robots. LLMs handle text, images, sounds, and videos increasingly well. This might seem to provide grounding through perceptual modality. However, as \textcite{jones_multimodal_2024} have shown, multimodal LLMs do not align with human grounding: they are sensitive to some pictorial features but diverge on affordances and orientation‑dependence, and other aspects of embodied perception. This is consistent with our claim that adding modalities enriches \textit{C} without providing a genuine $x: H\to W$.

More fundamentally, the data these LLMs process are not the rich, dynamic perceptual experience that grounds human experience. A picture of a coffee cup is a representation created by human agents using some devices. The LLM learns correlations between these representations and textual descriptions, but it still lacks the direct perceptual coupling that characterises human visual experience. Thus, in our categorical framework, multimodal training data remains within $C$ (or $C'$ after tokenisation/encoding). The LLM lacks a direct mapping to $W$ itself; instead, it has access to richer types of pre-grounded human content. This may seem to narrow the grounding gap, but it does not close it. And insofar as future LLMs or embodied systems that integrate vision, language, and motor control in real-time feedback loops with physical environments continue to remain `representational' and not `experiential' systems, no grounding will be available that is not human \parencite{bisk_experience_2020, incao_roadmap_2025}. This includes retrieval-augmented generation (RAG) architectures. The study by \textcite{li_oracleagent_2025} employs an advanced RAG architecture that, from an engineering standpoint, approximates the mathematical structure required by our framework. It ensures our reliability criterion, $P_{\mathrm{AI}}(h) \subseteq P_{\mathrm{human}}(h)$, by constraining the AI route’s outputs to originate from---or be verified against---the human route (i.e., the knowledge base $C$). Its findings empirically \textbf{support} our core philosophical claim: an LLM cannot achieve grounding on its own; it must rely indirectly on a human ``pre-grounded'' content repository $C$ (as implemented in the OracleAgent system) in order to function reliably.

In fact, \textcite{weihs_higher-order_2025} investigate the internal structure of $C$: the human knowledge base is not a simple ``set'' but rather a vast and intricate \textit{hypergraph}. The authors argue that even if an AI system succeeds in learning the \textit{higher-order structure} of $C$, it is still only learning a mathematical structure. This aligns with our view that learning does not constitute \textit{causal--perceptual grounding} in the world $W$. Meanwhile, \textcite{zur_are_2025} demonstrate that LLM ``reasoning'' operates as a \textit{branching-path} process. At each generated token, the AI navigates a probabilistic space filled with \textit{alternate paths}---unstable ``beliefs.'' This finding perfectly supports our argument: an LLM does not possess a stable, grounded ``belief'' (i.e., a specific state in $W$). What it has is merely the ability to navigate through a probabilistic space of branching possibilities. This is precisely what \textit{statistical pattern matching} really means.

Third, our analysis is related to debates about use-theoretic semantics \parencite{gubelmann_pragmatic_2024, wittgenstein_philosophical_1953}. Gubelmann contends that LLMs engage in some norm-governed linguistic practices through RLHF and alignment training, and thus possess meaning in Wittgenstein's `meaning as use' sense. However, our framework helps provide a more nuanced and hence correct approach. We can \textbf{distinguish} participating in norms (which requires experience, intentionality, and accountability) from being used within normative practices by human agents. Thus, one can acknowledge that LLMs demonstrate suitable linguistic behaviour---producing contextually plausible outputs, responding to feedback signals, and approximating human usage patterns---without asserting that they participate as actors in normative practices. The distinction is vital. Gubelmann refers to Brandom's inferentialist view, yet semantic inferentialism does require that an agent can give and ask for reasons, make commitments, and be held responsible, among other things. These are not merely behavioural tendencies but normative states that presuppose conscious agency and mental states. The crucial difference is that while RLHF might superficially resemble reason-giving---with reward signals functioning as a kind of feedback---this process lacks the essential features of genuine normative practice. The system does not understand the reward signal as a reason, cannot reflect on whether to accept or reject it, and bears no genuine responsibility for its outputs. When an LLM produces text that receives negative feedback during training, it adjusts its parameters, but this is a mechanistic optimisation rather than the kind of reflective endorsement or rejection that characterises experiential interactions. Any reward or punishment of an LLM for the right or wrong linguistic games it `played' would be meaningless. The responsibility for the system's output lies entirely with the human designers, operators, and users who deploy it, not with the system itself. When an LLM outputs `Paris is the capital of France,' it has not made a commitment that it is prepared to defend or for which it can be held responsible. It seems that LLMs may be a new kind of stochastic agents,\footnote{One of the authors has argued in favour of this hypothesis, see \parencite{floridi_ai_2025}.} but not the kind of epistemic agents that can ground symbols \parencite{barandiaran_transforming_2025}. Alice, who operates the system and interprets its output, carries whatever epistemic responsibility is associated with the claim. This is why LLMs cannot really (i.e., not metaphorically) lie like humans, but only hallucinate, and why we regard semantic interpretation ($r: O \to \mathrm{Pred}(W)$) as an external human act rather than an internal feature of an LLM. LLMs are systems used by human agents in linguistic practices, not participants within those practices. They are complex systems that human agents use to access, recombine, create, and present patterns from human-grounded content. This instrumental role aligns with their significant utility, but it also rules out genuine grounding.

In conclusion, LLMs' epistemic parasitism fosters impressive operational success. By recognising statistical patterns in human discourse, LLMs can produce outputs that, when interpreted by human users, often correspond with accurate information about the world. However, this success is fundamentally limited and derivative. The LLM's apparent semantic competence is not its own achievement, but an inheritance from the human authors whose grounded language fills its training corpus. Thus, outside $H^{\ast}$, where statistical patterns diverge from semantic truth, the LLM hallucinates, generating plausible strings that do not correspond to any proposition a reliable human epistemic process would produce. These failures are not bugs to be fixed, but structural features of a system that has circumvented rather than solved the grounding problem.

Understanding LLMs through this LoA---as sophisticated pattern-completion engines operating on pre-grounded content rather than as knowledge-bearing systems with semantic understanding---enables clearer expectations, more robust evaluation methodologies, and more responsible design and deployment practices. This approach respects Ockham's razor: we should not attribute grounding or understanding to a system when its behaviour can be adequately explained through statistical pattern-matching over human-produced content. Finally, it contextualises LLMs within a broader philosophical landscape: they are powerful informational tools that mediate human access to accumulated knowledge, but they are not themselves knowers.

\section{Limitations}
\label{sec:limitations}
Like any abstract analysis, the one we offer in this article has limitations. In this section, we discuss three that may be the most significant.

\subsection{Idealisations and Simplifications}

A possible critique is that this $\mathbf{Rel}$ framework makes some crucial idealisations and simplifications, especially in terms of processes, relations, and number of agents and LLMs involved. This is not unusual,\footnote{A more general and formal, categorical analysis of ideal agents and their interactions is provided by \textcite{tohme_category_2024} and \textcite{tohme_level-agnostic_2024}.} but we acknowledge that it is a methodological choice. For simplicity, we designate this (potentially messy) set of human output $P_{\mathrm{human}}(h)$ as the definitive \textbf{ground-truth benchmark} for that specific query $h$. Our soundness criterion $P_{\mathrm{AI}}(h) \subseteq P_{\mathrm{human}}(h)$ is therefore a \textit{relative} check against this human benchmark, not a check against some \textit{absolute}, \textit{external} truth (which is modelled by $W$). This methodological choice allows the human route to serve as the standard against which the LLM's output is evaluated for soundness. Most importantly, the validity of our results and arguments is independent of this choice.

\subsection{Semantics: Why Map to Propositions Instead of \texorpdfstring{$W$}{W}}
Both human and artificial routes map into $\mathrm{Pred}(W)$ rather than directly to $W$. This methodological choice reflects the fundamental point that neither texts nor LLM outputs \emph{are} the world; they are \emph{about} the world. As already mentioned, this separation ensures that we compare like with like: sets of propositions constructed from content versus sets of propositions constructed from outputs. It also prevents type errors: $C$ and $O$ consist of strings, $\mathrm{Pred}(W)$ consists of meanings, and $W$ consists of reality. Only within $\mathrm{Pred}(W)$ can we compare formally these proposition sets (e.g., via the set-inclusion relation $\subseteq$). Philosophically, this approach follows model-theoretic semantics: truth is evaluated in $W$, but meaning resides in $\mathrm{Pred}(W)$.
\subsection{Scope}

Our analysis focusses on offline, trained LLMs, which lack experiential access to $W$. Consequently, the presented framework is a `single-shot interaction model', which means it maps a specific query to a response rather than a multi-turn, state-updating conversational process. This restriction does not undermine our central thesis, particularly regarding the symbol grounding problem, as iterating ungrounded steps does not yield grounding. Furthermore, we observed that deployed systems, including agentic AI, may use tools (e.g., search, code execution, databases) or additional modalities (vision, audio), thereby broadening the available content at any given inference time. However, in our diagram, such tool outputs are still content: they enter through $C$ (and then $C'$), are processed via $e$, and are interpreted via $r$. This does not, by itself, resolve symbol grounding; it merely expands the consulted content. We acknowledge the characterisation of humans as agents, limited explicitly at a particular background knowledge and epistemic goals, as an exemplifying tool, yet an inherently reductive schema, as are all tools of abstraction.

\section{Conclusion}
\label{sec:conclusion}
This theoretical framework offers a systematic comparison of human epistemic routes and computational artificial routes from informational situations to propositions about the world. Abstracting away from low-level implementational details helps clarify the conditions under which AI--human epistemic alignment occurs. A shared formal understanding is a prerequisite to resolving these alignment issues.
%An agreement is more likely if we can share a common understanding of the issues at hand.

When the \textbf{entailment criterion holds} ($h \in H^{\ast}$), the AI path's output set is contained within the human path's set. Those situations constitute $H^{\ast}$, the soundness wherein the LLM functions as a sound (if not necessarily complete) informational source and a reliable interface to human content. Outside $H^{\ast}$, failures arise from systematic mismatches in tokenisation, dataset construction, training, prompting, inference, or interpretation.

This analytical perspective presents LLMs as very powerful yet fundamentally limited mechanisms: they manipulate symbols through statistical processes and can align with human knowledge when appropriately trained and deployed \parencite{floridi_ai_2023}. The framework supports both philosophical clarity and practical alignment research: success (soundness) means that AI's output set $P_{\mathrm{AI}}(h)$ is contained within the human ground-truth set ($P_{\mathrm{AI}}(h) \subseteq P_{\mathrm{human}}(h)$), ensuring it produces no propositions that lack grounding in reliable content \parencite{russell_human_2019, lin_truthfulqa_2022}.

Crucially, this analysis shows that LLMs, lacking any experiential access to the world, possess only mediated access to information \textit{about} the world through representations authored by humans. They do not solve the symbol grounding problem; instead, they circumvent it by exploiting content produced by humans who have already resolved the grounding issue. Even without a genuine understanding, LLMs can function as effective informational intermediaries. However, responsible deployment requires systematic awareness of the boundaries of $H^{\ast}$: trusting these systems only within domains where soundness (entailment) demonstrably holds, and implementing constraints or verification procedures elsewhere \parencite{christian_alignment_2020}. Hallucinations represent systematic failures of entailment, that is, instances where statistical pattern completion produces propositions that lack grounding in the factual content that humans would construct through reliable epistemic processes. They are part of the design of LLMs, an unfortunate feature, not an avoidable bug. They must be managed intelligently because they cannot be technically eradicated.

\section*{References}

\begingroup
\setlength{\parindent}{0pt}
\setlength{\parskip}{4pt}
\sloppy
\printbibliography[heading=none]

@article{mollo_vector_2025,
	title = {The vector grounding problem},
	urldate = {2025-11-09},
	journal = {arXiv preprint arXiv:2304.01481},
	author = {Mollo, Dimitri Coelho and Millière, Raphaël},
	year = {2025},
	note = {arXiv preprint},
}

@book{marr_vision_1982,
	title = {Vision: {A} computational investigation into the human representation and processing of visual information},
	publisher = {W. H. Freeman},
	author = {Marr, David},
	year = {1982},
}

@article{montague_universal_1970,
	title = {Universal grammar},
	volume = {36},
	doi = {10.1111/j.1755-2567.1970.tb00434.x},
	number = {3},
	journal = {Theoria},
	author = {Montague, Richard},
	year = {1970},
	pages = {373--398},
}

@article{barandiaran_transforming_2025,
	title = {Transforming agency: {On} the mode of existence of large language models},
	doi = {10.1007/s11097-025-10094-3},
	journal = {Phenomenology and the Cognitive Sciences},
	author = {Barandiaran, Xabier E. and Almendros, Lola S.},
	month = aug,
	year = {2025},
}

@inproceedings{lin_truthfulqa_2022,
	title = {{TruthfulQA}: {Measuring} how models mimic human falsehoods},
	doi = {10.18653/v1/2022.acl-long.229},
	booktitle = {Proceedings of the 60th {Annual} {Meeting} of the {Association} for {Computational} {Linguistics} ({Volume} 1: {Long} {Papers})},
	publisher = {Association for Computational Linguistics},
	author = {Lin, Stephanie and Hilton, Jacob and Evans, Owain},
	year = {2022},
	pages = {3214--3252},
}

@article{harnad_symbol_1990,
	title = {The symbol grounding problem},
	volume = {42},
	doi = {10.1016/0167-2789(90)90087-6},
	number = {1–3},
	journal = {Physica D: Nonlinear Phenomena},
	author = {Harnad, Stevan},
	year = {1990},
	pages = {335--346},
}

@article{tarski_semantic_1944,
	title = {The semantic conception of truth and the foundations of semantics},
	volume = {4},
	doi = {10.2307/2102968},
	number = {3},
	journal = {Philosophy and Phenomenological Research},
	author = {Tarski, Alfred},
	year = {1944},
	pages = {341--376},
}

@misc{marcus_road_2022,
	title = {The road to {AI} we can trust},
	url = {https://garymarcus.substack.com},
	urldate = {2025-11-09},
	journal = {Substack},
	author = {Marcus, Gary},
	year = {2022},
	note = {Substack newsletter},
}

@book{floridi_philosophy_2011,
	title = {The philosophy of information},
	publisher = {Oxford University Press},
	author = {Floridi, Luciano},
	year = {2011},
}

@book{shannon_mathematical_1949,
	title = {The mathematical theory of communication},
	publisher = {University of Illinois Press},
	author = {Shannon, Claude E. and Weaver, Warren},
	year = {1949},
}

@article{floridi_method_2008,
	title = {The method of levels of abstraction},
	volume = {18},
	doi = {10.1007/s11023-008-9113-7},
	number = {3},
	journal = {Minds and Machines},
	author = {Floridi, Luciano},
	year = {2008},
	pages = {303--329},
}

@article{clark_extended_1998,
	title = {The extended mind},
	volume = {58},
	doi = {10.1093/analys/58.1.7},
	number = {1},
	journal = {Analysis},
	author = {Clark, Andy and Chalmers, David},
	year = {1998},
	pages = {7--19},
}

@article{newell_knowledge_1982,
	title = {The knowledge level},
	volume = {18},
	doi = {10.1016/0004-3702(82)90012-1},
	number = {1},
	journal = {Artificial Intelligence},
	author = {Newell, Allen},
	year = {1982},
	pages = {87--127},
}

@article{wu_mechanistic_2025,
	title = {The mechanistic emergence of symbol grounding in language models},
	doi = {10.48550/arXiv.2510.13796},
	urldate = {2025-11-09},
	journal = {arXiv},
	author = {Wu, Shuyu and Ma, Ziqiao and Luo, Xiaoxi and Huang, Yidong and Torres-Fonseca, Josue and Shi, Freda and Chai, Joyce},
	year = {2025},
}

@book{christian_alignment_2020,
	title = {The alignment problem: {Machine} learning and human values},
	publisher = {W. W. Norton \& Company},
	author = {Christian, Brian},
	year = {2020},
}

@article{gallese_brains_2005,
	title = {The brain's concepts: {The} role of the sensory-motor system in conceptual knowledge},
	volume = {22},
	doi = {10.1080/02643290442000310},
	number = {3–4},
	journal = {Cognitive Neuropsychology},
	author = {Gallese, Vittorio and Lakoff, George},
	year = {2005},
	pages = {455--479},
}

@article{dove_symbol_2024,
	title = {Symbol ungrounding: {What} the successes (and failures) of large language models reveal about human cognition},
	volume = {379},
	doi = {10.1098/rstb.2023.0149},
	number = {1911},
	journal = {Philosophical Transactions of the Royal Society B: Biological Sciences},
	author = {Dove, Guy},
	year = {2024},
	pages = {20230149},
}

@article{ji_survey_2023,
	title = {Survey of hallucination in natural language generation},
	volume = {55},
	doi = {10.1145/3571730},
	number = {12},
	journal = {ACM Computing Surveys},
	author = {Ji, Ziwei and Lee, Nayeon and Frieske, Rita and Yu, Tiezheng and Su, Dan and Xu, Yan and Ishii, Etsuko and Bang, Ye Jin and Madotto, Andrea and Fung, Pascale},
	year = {2023},
	pages = {1--38},
}

@article{pavlick_symbols_2023,
	title = {Symbols and grounding in large language models},
	volume = {381},
	doi = {10.1098/rsta.2022.0041},
	number = {2251},
	journal = {Philosophical Transactions of the Royal Society A},
	author = {Pavlick, Ellie},
	year = {2023},
	pages = {20220041},
}

@book{rosiak_sheaf_2022,
	title = {Sheaf theory through examples},
	isbn = {0-262-54215-3},
	publisher = {MIT Press},
	author = {Rosiak, Daniel},
	year = {2022},
}

@book{perrone_starting_2024,
	title = {Starting category theory},
	publisher = {World Scientific},
	author = {Perrone, Paolo},
	year = {2024},
}

@article{taddeo_solving_2005,
	title = {Solving the symbol grounding problem: {A} critical review of fifteen years of research},
	volume = {17},
	doi = {10.1080/09528130500283655},
	number = {4},
	journal = {Journal of Experimental \& Theoretical Artificial Intelligence},
	author = {Taddeo, Mariarosaria and Floridi, Luciano},
	year = {2005},
	pages = {419--445},
}

@book{fodor_psychosemantics_1987,
	title = {Psychosemantics: {The} problem of meaning in the philosophy of mind},
	publisher = {MIT Press},
	author = {Fodor, Jerry A.},
	year = {1987},
}

@article{ye_provable_2025,
	title = {Provable separations between memorization and generalization in diffusion models},
	doi = {10.48550/arXiv.2511.03202},
	urldate = {2025-11-09},
	journal = {arXiv},
	author = {Ye, Zeqi and Zhu, Qijie and Tao, Molei and Chen, Minshuo},
	year = {2025},
}

@article{floridi_semantic_2011,
	title = {Semantic information and the correctness theory of truth},
	volume = {74},
	doi = {10.1007/s10670-010-9249-8},
	number = {2},
	journal = {Erkenntnis},
	author = {Floridi, Luciano},
	year = {2011},
	pages = {147--175},
}

@inproceedings{gubelmann_pragmatic_2024,
	title = {Pragmatic norms are all you need—{Why} the symbol grounding problem does not apply to {LLMs}},
	doi = {10.18653/v1/2024.emnlp-main.651},
	booktitle = {Proceedings of the 2024 {Conference} on {Empirical} {Methods} in {Natural} {Language} {Processing}},
	publisher = {Association for Computational Linguistics},
	author = {Gubelmann, Reto},
	year = {2024},
	pages = {11663--11678},
}

@book{wittgenstein_philosophical_1953,
	title = {Philosophical investigations},
	publisher = {Blackwell},
	author = {Wittgenstein, Ludwig},
	translator = {Anscombe, G. E. M.},
	year = {1953},
}

@book{quine_philosophy_1970,
	title = {Philosophy of logic},
	publisher = {Prentice Hall},
	author = {Quine, Willard Van Orman},
	year = {1970},
}

@article{barsalou_perceptual_1999,
	title = {Perceptual symbol systems},
	volume = {22},
	doi = {10.1017/S0140525X99002149},
	number = {4},
	journal = {Behavioral and Brain Sciences},
	author = {Barsalou, Lawrence W.},
	year = {1999},
	pages = {577--660},
}

@article{li_oracleagent_2025,
	title = {{OracleAgent}: {A} multimodal reasoning agent for oracle bone script research},
	doi = {10.48550/arXiv.2510.26114},
	urldate = {2025-11-09},
	journal = {arXiv},
	author = {Li, Caoshuo and Ding, Zengmao and Hu, Xiaobin and Li, Bang and Luo, Donghao and Peng, Xu and Jin, Taisong and Liu, Yongge and Han, Shengwei and Yang, Jing and He, Xiaoping and Gao, Feng and Wu, AndyPian and {SevenShu} and Wang, Chaoyang and Wang, Chengjie},
	year = {2025},
}

@article{perrone_notes_2019,
	title = {Notes on {Category} {Theory} with examples from basic mathematics},
	journal = {arXiv preprint arXiv:1912.10642},
	author = {Perrone, Paolo},
	year = {2019},
}

@inproceedings{bender_dangers_2021,
	title = {On the dangers of stochastic parrots: {Can} language models be too big?},
	doi = {10.1145/3442188.3445922},
	booktitle = {Proceedings of the 2021 {ACM} {Conference} on {Fairness}, {Accountability}, and {Transparency} ({FAccT})},
	publisher = {ACM},
	author = {Bender, Emily M. and Gebru, Timnit and McMillan-Major, Angelina and Shmitchell, Shmargaret},
	year = {2021},
	pages = {610--623},
}

@article{searle_minds_1980,
	title = {Minds, brains, and programs},
	volume = {3},
	doi = {10.1017/S0140525X00005756},
	number = {3},
	journal = {Behavioral and Brain Sciences},
	author = {Searle, John R.},
	year = {1980},
	pages = {417--457},
}

@book{kripke_naming_1980,
	title = {Naming and necessity},
	publisher = {Harvard University Press},
	author = {Kripke, Saul A.},
	year = {1980},
}

@article{patterson_knowledge_2017,
	title = {Knowledge representation in bicategories of relations},
	journal = {arXiv preprint arXiv:1706.00526},
	author = {Patterson, Evan},
	year = {2017},
}

@book{brandom_making_1994,
	title = {Making it explicit: {Reasoning}, representing, and discursive commitment},
	publisher = {Harvard University Press},
	author = {Brandom, Robert B.},
	year = {1994},
}

@article{piantadosi_meaning_2022,
	title = {Meaning without reference in large language models},
	doi = {10.48550/arXiv.2208.02957},
	urldate = {2025-11-09},
	journal = {arXiv},
	author = {Piantadosi, Steven T. and Hill, Felix},
	year = {2022},
}

@misc{maliugina_llm_2024,
	title = {{LLM} hallucinations and failures: {Lessons} from 5 examples},
	url = {https://www.evidentlyai.com/blog/llm-hallucination-examples},
	urldate = {2025-11-09},
	journal = {Evidently AI Blog},
	author = {Maliugina, Dasha},
	month = sep,
	year = {2024},
}

@article{tohme_level-agnostic_2024,
	title = {Level-agnostic representations of interacting agents},
	volume = {12},
	doi = {10.3390/math12172697},
	number = {17},
	journal = {Mathematics},
	author = {Tohmé, Fernando and Fioriti, Andrés},
	year = {2024},
	pages = {2697},
}

@article{harnad_language_2024,
	title = {Language writ large: {LLMs}, {ChatGPT}, meaning, and understanding},
	volume = {7},
	doi = {10.3389/frai.2024.1490698},
	journal = {Frontiers in Artificial Intelligence},
	author = {Harnad, Stevan},
	year = {2024},
	pages = {1490698},
}

@book{russell_human_2019,
	title = {Human compatible: {Artificial} intelligence and the problem of control},
	publisher = {Viking},
	author = {Russell, Stuart},
	year = {2019},
}

@book{dretske_knowledge_1981,
	title = {Knowledge and the flow of information},
	publisher = {MIT Press},
	author = {Dretske, Fred I.},
	year = {1981},
}

@inproceedings{carta_grounding_2023,
	title = {Grounding large language models in interactive environments with online reinforcement learning},
	url = {https://proceedings.mlr.press/v202/carta23a.html},
	urldate = {2025-11-09},
	booktitle = {Proceedings of the 40th {International} {Conference} on {Machine} {Learning} ({ICML} 2023)},
	publisher = {PMLR},
	author = {Carta, Thomas and Romac, Clément and Wolf, Thomas and Lamprier, Sylvain and Sigaud, Olivier and Oudeyer, Pierre-Yves},
	year = {2023},
	pages = {3676--3713},
}

@article{weihs_higher-order_2025,
	title = {Higher-order regularization learning on hypergraphs},
	doi = {10.48550/arXiv.2510.26533},
	urldate = {2025-11-09},
	journal = {arXiv},
	author = {Weihs, Adrien and Bertozzi, Andrea and Thorpe, Matthew},
	year = {2025},
}

@article{openai_gpt-4_2023,
	title = {{GPT}-4 technical report},
	doi = {10.48550/arXiv.2303.08774},
	urldate = {2025-11-09},
	journal = {arXiv},
	author = {{OpenAI}},
	year = {2023},
}

@article{bai_forget_2025,
	title = {Forget {BIT}, it is all about {TOKEN}: {Towards} semantic information theory for {LLMs}},
	doi = {10.48550/arXiv.2511.01202},
	urldate = {2025-11-09},
	journal = {arXiv},
	author = {Bai, Bo},
	year = {2025},
}

@incollection{sellars_empiricism_1956,
	title = {Empiricism and the philosophy of mind},
	volume = {1},
	booktitle = {Minnesota studies in the philosophy of science},
	publisher = {University of Minnesota Press},
	author = {Sellars, Wilfrid},
	editor = {Feigl, Herbert and Scriven, Michael},
	year = {1956},
	pages = {253--329},
}

@inproceedings{bisk_experience_2020,
	title = {Experience grounds language},
	doi = {10.18653/v1/2020.emnlp-main.703},
	booktitle = {Proceedings of the 2020 {Conference} on {Empirical} {Methods} in {Natural} {Language} {Processing} ({EMNLP})},
	publisher = {Association for Computational Linguistics},
	author = {Bisk, Yonatan and Holtzman, Ari and Thomason, Jesse and Andreas, Jacob and Bengio, Yoshua and Chai, Joyce and Lapata, Mirella and Lazaridou, Angeliki and Zettlemoyer, Luke and May, Jonathan and Nisnevich, Aleksandr and Pinto, Nicolas and Turian, Joseph},
	year = {2020},
	pages = {8718--8735},
}

@book{spivak_category_2014,
	title = {Category theory for the sciences},
	isbn = {0-262-32053-3},
	publisher = {MIT press},
	author = {Spivak, David I.},
	year = {2014},
}

@article{brown_categories_2009,
	title = {Categories of timed stochastic relations},
	volume = {249},
	doi = {10.1016/j.entcs.2009.07.091},
	journal = {Electronic Notes in Theoretical Computer Science},
	author = {Brown, Daniel and Pucella, Riccardo},
	year = {2009},
	pages = {193--217},
}

@book{landry_categories_2017,
	title = {Categories for the working philosopher},
	publisher = {Oxford University Press},
	author = {Landry, Elaine M.},
	year = {2017},
}

@article{jones_multimodal_2024,
	title = {Do multimodal large language models and humans ground language similarly?},
	volume = {50},
	doi = {10.1162/coli_a_00531},
	number = {4},
	journal = {Computational Linguistics},
	author = {Jones, Cameron R. and Bergen, Benjamin K. and Trott, Sean},
	year = {2024},
	pages = {1415--1440},
}

@article{jia_category-theoretical_2024,
	title = {Category-theoretical and topos-theoretical frameworks in machine learning: {A} survey},
	doi = {10.48550/arXiv.2408.14014},
	urldate = {2025-11-09},
	journal = {arXiv},
	author = {Jia, Yiyang and Peng, Guohong and Yang, Zheng and Chen, Tianhao},
	year = {2024},
}

@book{fong_invitation_2019,
	title = {An invitation to applied category theory: seven sketches in compositionality},
	isbn = {1-108-58224-9},
	publisher = {Cambridge University Press},
	author = {Fong, Brendan and Spivak, David I.},
	year = {2019},
}

@article{kornell_axioms_2023,
	title = {Axioms for the category of sets and relations},
	journal = {arXiv preprint arXiv:2302.14153},
	author = {Kornell, Andre},
	year = {2023},
}

@article{floridi_anthropomorphising_2024,
	title = {Anthropomorphising machines and computerising minds: {The} crosswiring of languages between artificial intelligence and brain \& cognitive sciences},
	volume = {34},
	doi = {10.1007/s11023-024-09670-4},
	number = {1},
	journal = {Minds and Machines},
	author = {Floridi, Luciano and Nobre, Anna C.},
	month = apr,
	year = {2024},
	pages = {5},
}

@article{zur_are_2025,
	title = {Are language models aware of the road not taken? {Token}-level uncertainty and hidden state dynamics},
	doi = {10.48550/arXiv.2511.04527},
	urldate = {2025-11-09},
	journal = {arXiv},
	author = {Zur, Amir and Geiger, Atticus and Lubana, Ekdeep Singh and Bigelow, Eric},
	year = {2025},
}

@article{shiebler_categorical_2021,
	title = {Categorical stochastic processes and likelihood},
	volume = {3},
	doi = {10.32408/compositionality-3-1},
	journal = {Compositionality},
	author = {Shiebler, Dan},
	year = {2021},
	pages = {1},
}

@article{floridi_ai_2023,
	title = {{AI} as agency without intelligence: {On} {ChatGPT}, large language models, and other generative models},
	volume = {36},
	doi = {10.1007/s13347-023-00611-7},
	journal = {Philosophy \& Technology},
	author = {Floridi, Luciano},
	year = {2023},
	pages = {1--15},
}

@article{floridi_ai_2025,
	title = {{AI} as agency without intelligence: {On} artificial intelligence as a new form of artificial agency and the multiple realisability of agency thesis},
	volume = {38},
	doi = {10.1007/s13347-025-00858-9},
	number = {1},
	journal = {Philosophy \& Technology},
	author = {Floridi, Luciano},
	month = feb,
	year = {2025},
	pages = {30},
}

@incollection{borghi_action_2014,
	title = {Action and language integration: {From} humans to cognitive robots},
	booktitle = {Handbook of cognitive science: {An} embodied approach},
	publisher = {Elsevier},
	author = {Borghi, Anna M. and Cangelosi, Angelo},
	editor = {Calvo, Paco and Gomila, Toni},
	year = {2014},
	pages = {289--314},
}

@article{rawte_survey_2023,
	title = {A survey of hallucination in large foundation models},
	doi = {10.48550/arXiv.2309.05922},
	urldate = {2025-11-09},
	journal = {arXiv},
	author = {Rawte, Vipula and Sheth, Amit and Das, Amitava},
	year = {2023},
}

@article{qi_survey_2024,
	title = {A survey of automatic hallucination evaluation on natural language generation},
	doi = {10.48550/arXiv.2404.12041},
	urldate = {2025-11-09},
	journal = {arXiv},
	author = {Qi, Siya and Gui, Lin and He, Yulan and Yuan, Zheng},
	year = {2024},
}

@article{mahadevan_rose_2025,
	title = {A rose by any other name would smell as sweet: {Categorical} homotopy theory for large language models},
	doi = {10.48550/arXiv.2508.10018},
	urldate = {2025-11-09},
	journal = {arXiv},
	author = {Mahadevan, Sridhar},
	year = {2025},
}

@inproceedings{incao_roadmap_2025,
	title = {A roadmap for embodied and social grounding in {LLMs}},
	doi = {10.3233/FAIA241488},
	booktitle = {Social {Robots} with {AI}: {Prospects}, {Risks}, and {Responsible} {Methods} ({Robophilosophy} 2024)},
	publisher = {IOS Press (Frontiers in Artificial Intelligence and Applications, Vol. 397)},
	author = {Incao, Sara and Mazzola, Carlo and Belgiovine, Giulia and Sciutti, Alessandra},
	year = {2025},
	pages = {43--52},
}

@article{tohme_category_2024,
	title = {A category theory approach to the semiotics of machine learning},
	volume = {92},
	doi = {10.1007/s10472-024-09932-y},
	number = {3},
	journal = {Annals of Mathematics and Artificial Intelligence},
	author = {Tohmé, Fernando and Gangle, Rocco and Caterina, Gianluca},
	year = {2024},
	pages = {733--751},
}

@article{borji_categorical_2023,
	title = {A categorical archive of {ChatGPT} failures},
	doi = {10.48550/arXiv.2302.03494},
	urldate = {2025-11-09},
	journal = {arXiv},
	author = {Borji, Ali},
	year = {2023},
}

@article{taddeo_praxical_2007,
	title = {A praxical solution of the symbol grounding problem},
	volume = {17},
	doi = {10.1007/s11023-007-9073-3},
	number = {4},
	journal = {Minds and Machines},
	author = {Taddeo, Mariarosaria and Floridi, Luciano},
	year = {2007},
	pages = {369--389},
}

@book{awodey_category_2010,
	title = {Category theory},
	publisher = {Oxford University Press},
	author = {Awodey, Steve},
	year = {2010},
}
\endgroup

\appendix

\section{Appendix: Category Theory and \texorpdfstring{$\mathbf{Rel}$}{Rel}}
\label{sec:appendix}

\noindent We provide here the necessary mathematical definitions to interpret the formal framework presented in the article. We limit our attention to the general definition of a category and the specific properties of the category of relations $\mathbf{Rel}$, which serves as the ambient mathematical structure for our analysis of LLM grounding.

\vspace{1em}

\noindent \textbf{Definition of a Category.} A category $\mathbf{C}$ consists of a collection of \textbf{objects}, denoted $\mathrm{Ob}(\mathbf{C})$, and for every pair of objects $A, B$, a set of \textbf{morphisms} (or arrows) denoted $\mathrm{Hom}(A, B)$. If $f$ is a morphism from $A$ to $B$, we write $f: A \to B$. The category is equipped with a composition operation: for any morphisms $f: A \to B$ and $g: B \to C$, there exists a composite morphism $g \circ f: A \to C$. This structure must satisfy two axioms:
\begin{enumerate}
    \item \textbf{Associativity:} For all $f: A \to B$, $g: B \to C$, and $h: C \to D$, the equation $(h \circ g) \circ f = h \circ (g \circ f)$ holds.
    \item \textbf{Identity:} For every object $A$, there exists an identity morphism $\mathrm{id}_A: A \to A$ such that for every $f: A \to B$, $f \circ \mathrm{id}_A = f$ and $\mathrm{id}_B \circ f = f$.
\end{enumerate}
\noindent In standard category theory, a diagram is said to \textbf{commute} if the composition of morphisms along any two paths with the same start and end points yields the \textit{equal} morphism (e.g., $g \circ f = h$).

\vspace{1em}

\noindent \textbf{The Category $\mathbf{Rel}$.} This is the category of sets and relations, with the following components:
\begin{itemize}
    \item \textbf{Objects:} The objects of $\mathbf{Rel}$ are sets (e.g., the set of epistemic situations $H$, the set of possible worlds $W$).
    \item \textbf{Morphisms:} A morphism $R: A \to B$ is a binary relation, defined as a subset of the Cartesian product, $R \subseteq A \times B$.
    \item \textbf{Identity:} The identity morphism on a set $A$ is the diagonal relation, $\mathrm{id}_A = \{ (a, a) \mid a \in A \}$.
    \item \textbf{Composition:} Given two relations $R: A \to B$ and $S: B \to C$, their composition $S \circ R: A \to C$ is defined by:
    \[
    S \circ R = \{ (a, c) \in A \times C \mid \exists b \in B \text{ such that } (a, b) \in R \text{ and } (b, c) \in S \}.
    \]
\end{itemize}
\noindent Unlike the category of sets and functions $\mathbf{Set}$, morphisms in $\mathbf{Rel}$ can be ``one-to-many'', relating a single input to a set of outputs (which can be empty).

\vspace{1em}

\noindent \textbf{Order Enriched Structure.} A crucial feature of $\mathbf{Rel}$ used in this paper is that it is an \textit{ordered category}. For any two objects $A$ and $B$, the set of morphisms $\mathrm{Hom}(A, B)$ is partially ordered by subset inclusion ($\subseteq$).
Given two relations $R, S: A \to B$, we say that $R$ entails $S$ (denoted $R \subseteq S$) if $R$ is a subset of $S$ as sets of pairs. This order is compatible with composition: if $R \subseteq R'$ and $S \subseteq S'$, then $S \circ R \subseteq S' \circ R'$.

\vspace{1em}

\noindent \textbf{Entailment-Commutativity (Lax Commutativity).} Due to the ordered structure of $\mathbf{Rel}$, the standard notion of diagram equality is relaxed to inclusion. In this paper, we define success not by strict equality but by the existence of an inequality ($2$-morphism) between the paths. A diagram exhibits \textbf{entailment-commutativity} if the composite morphism of the lower path is contained within the composite morphism of the upper path.

\vspace{1em}

\noindent \textbf{Universal Properties, Limits, and Colimits.} A \textbf{universal property} characterises an object (or a morphism) by its relationship to all other objects of a certain type. This concept is central to defining \textbf{limits} and \textbf{colimits}.\\

\noindent A \textbf{limit} (e.g., the Cartesian product, intersection, or kernel) is a universal object that maps \textit{to} a diagram of objects. Dually, a \textbf{colimit} (e.g., the disjoint union, union, or quotient) is a universal object that maps \textit{out} of a diagram.

\vspace{1em}

%\noindent \textcolor{red}{\textbf{Limits and Colimits in Rel.} The category \textbf{Rel} exhibits unique structural properties regarding limits:}
%\begin{itemize}
 %   \item \textcolor{red}{\textbf{Categorical Limits (Biproducts):} Unlike in $\mathbf{Set}$, where the product is the Cartesian pair and the coproduct is the disjoint union, in $\mathbf{Rel}$, the categorical product and coproduct coincide. They are both given by the \textbf{disjoint union} of sets ($A \sqcup B$).}
    
  %  \item \textcolor{red}{\textbf{Order-Theoretic Limits (Unions and Intersections):} Because $\mathbf{Rel}$ is enriched over posets (ordered by inclusion), we can also consider limits and colimits within the hom-sets (the set of relations between fixed objects). In this local context:}
  %  \begin{itemize}
   %     \item \textcolor{red}{The \textbf{colimit} (supremum) of a family of relations is their \textbf{union} ($\bigcup R_i$). This operation is used to define the Right Kan Extension in our framework (the ``largest'' sound relation).}
    %    \item \textcolor{red}{The \textbf{limit} (infimum) of a family of relations is their \textbf{intersection} ($\bigcap R_i$).}
   % \end{itemize}
%\end{itemize}

\noindent \textbf{Limits and Colimits in $\mathbf{Rel}$.} The category $\mathbf{Rel}$ exhibits unique structural properties regarding limits:
\begin{itemize}
    \item \textbf{Categorical Limits (Biproducts):} Unlike in $\mathbf{Set}$, where the product is the Cartesian pair and the coproduct is the disjoint union, in $\mathbf{Rel}$, the categorical product and coproduct coincide. They are both given by the \textbf{disjoint union} of sets ($A \sqcup B$).
    
  %  \item \textbf{Order-Theoretic Limits (Unions and Intersections):} Because $\mathbf{Rel}$ is enriched over posets (ordered by inclusion), we can also consider limits and colimits within the hom-sets (the set of relations between fixed objects). In this local context:
  %  \begin{itemize}
   %     \item The \textbf{colimit} (supremum) of a family of relations is their \textbf{union} ($\bigcup R_i$). This corresponds to the \textbf{Left Kan Extension}.
   %     \item The \textbf{limit} (infimum) of a family of relations is their \textbf{intersection} ($\bigcap R_i$). This operation is used to define the \textbf{Right Kan Extension} in our framework (representing the maximal relation entailed by the ground truth, ensuring soundness).
  %  \end{itemize}

    \item \textbf{Order-Theoretic Operations:} Because $\mathbf{Rel}$ is enriched over posets (ordered by inclusion), each hom-set $\mathrm{Hom}(A, B)$ forms a complete lattice. In this local context:
    \begin{itemize}
        \item The \textbf{supremum} of a family of relations is their \textbf{union} ($\bigcup R_i$).
        \item The \textbf{infimum} of a family of relations is their \textbf{intersection} ($\bigcap R_i$).
    \end{itemize}

\vspace{1em}

    \noindent \textbf{Note on the Right Kan Extension:} In our framework, the Right Kan Extension $\mathrm{Ran}_p(g \circ c)$ exhibits a dual nature depending on the perspective:
    \begin{itemize}
        \item \textbf{Globally (as a Relation):} It is defined as the \textbf{supremum} (union) of all valid sound relations:
        \[ \mathrm{Ran}_p(g \circ c) = \bigcup \{ R \subseteq C' \times \mathrm{Pred}(W) \mid R \circ p \subseteq g \circ c \}. \]
        This construction yields the \textit{maximal} relation that satisfies the soundness constraint.    
        \item \textbf{Pointwise (as a Value):} For any specific input $c'$, the value of this maximal relation is computed as the \textbf{intersection} of the constraints:
        \[ \mathrm{Ran}_p(g \circ c)(c') = \bigcap_{h \in p^{-1}(c')} (g \circ c)(h). \]
        This ensures that the output is true in \textit{all} possible interpretations, thereby guaranteeing soundness.
\end{itemize}
\end{itemize}
\end{document}